\documentclass{article}

\usepackage{microtype}
\usepackage{graphicx}
\usepackage{booktabs} 

\usepackage{hyperref}

\usepackage[accepted]{icml2019}

\icmltitlerunning{Partially Linear Additive Gaussian Graphical Models}
\usepackage{enumitem}

\usepackage{amsmath,amsthm,amssymb}
\allowdisplaybreaks
\usepackage{caption}
\usepackage{subcaption}

\usepackage{bm}

\usepackage{capt-of}

\usepackage{hyperref}
\usepackage{natbib}

\newcommand{\grad}{\bm{\nabla}}

\newcommand{\norm}[1]{\lVert #1 \rVert}
\newcommand{\Norm}[1]{\left\lVert #1 \right\rVert}
\usepackage{mathtools}
\DeclarePairedDelimiter\abs{\lvert}{\rvert}%
\makeatletter
\let\oldabs\abs
\def\abs{\@ifstar{\oldabs}{\oldabs*}}

\makeatother

\newcommand{\bTheta}{\bm{\Theta}}
\newcommand{\bx}{\mathbf{x}}

\newcommand{\bz}{\mathbf{z}}
\newcommand{\by}{\mathbf{y}}

\newcommand{\bX}{\mathbf{X}}
\newcommand{\bOmega}{\mathbf{\Omega}}
\newcommand{\bzero}{\mathbf{0}}
\newcommand{\bZ}{\mathbf{Z}}
\newcommand{\bM}{\mathbf{M}}
\newcommand{\bH}{\mathbf{H}}
\newcommand{\bD}{\mathbf{D}}
\newcommand{\bR}{\mathbf{R}}
\newcommand{\bB}{\mathbf{B}}

\newcommand{\bb}{\mathbf{b}}
\newcommand{\bW}{\mathbf{W}}

\newcommand{\bS}{\mathbf{S}}
\newcommand{\bO}{\mathbf{O}}
\newcommand{\bI}{\mathbf{I}}
\newcommand{\bbeta}{\bm{\beta}}
\newcommand{\balpha}{\bm{\alpha}}
\newcommand{\bY}{\mathbf{Y}}
\newcommand{\bSigma}{\mathbf{\Sigma}}

\newcommand{\bDelta}{\bm{\Delta}}

\newcommand{\bepsilon}{\bm{\epsilon}}

\newcommand{\EE}{\mathbb{E}} 

\newcommand\given[1][]{\:#1\vert\:}
\DeclareMathOperator*{\argmin}{arg\,min}

\usepackage{physics}

\usepackage{dsfont}

\newtheorem{lemma}{Lemma}
\newtheorem{theorem}{Theorem}

\newtheorem{definition}{Definition}

\newtheorem{assumption}{Assumption}

\usepackage{graphicx}

\usepackage{multirow}
\usepackage{float}

\usepackage{comment}
\newcommand{\sinong}[1]{\textcolor{blue}{[Sinong: #1]}}
\newcommand{\sanmi}[1]{\textcolor{red}{[Sanmi: #1]}}

\begin{document}
	
	\twocolumn[
	\icmltitle{Partially Linear Additive Gaussian Graphical Models}
	
	
	
	\icmlsetsymbol{equal}{*}
	
	\begin{icmlauthorlist}
		\icmlauthor{Sinong Geng}{pu}
		\icmlauthor{Minhao Yan}{cu}
		\icmlauthor{Mladen Kolar}{uc}
		\icmlauthor{Oluwasanmi Koyejo}{uiuc}
	\end{icmlauthorlist}
	
	\icmlaffiliation{pu}{Department of Computer Science, Princeton University, Princeton, New Jersey, USA}
	\icmlaffiliation{cu}{Charles H. Dyson School of Applied Economics and Management, Ithaca, New York, USA}
	\icmlaffiliation{uc}{                Booth School of Business, University of Chicago, Chicago, Illinois, USA}
	\icmlaffiliation{uiuc}{                Department of Computer Science, University of Illinois at Urbana-Champaign, University of Illinois Urbana-Champaign, Champaign, Illinois, USA}
	\icmlcorrespondingauthor{Sinong Geng}{sgeng@princeton.edu}

	\icmlkeywords{Machine Learning, ICML}
	
	\vskip 0.3in
	]
	
	
	
	\printAffiliationsAndNotice{}  

%

%

\begin{abstract}
We propose a partially linear additive Gaussian graphical model (PLA-GGM) for the estimation of
associations between random variables distorted by observed confounders.
Model parameters are estimated 
using an $L_1$-regularized maximal pseudo-profile likelihood estimator (MaPPLE) for which we
prove $\sqrt{n}$-sparsistency.
Importantly, our approach avoids parametric constraints on the effects of confounders on the estimated graphical model structure. 
Empirically, the PLA-GGM is applied to both synthetic and real-world datasets, demonstrating superior performance compared to competing methods.

\end{abstract}

\section{Introduction}
Undirected graphical models
are extensively used to study the conditional independence structure
between random variables~\citep{jordan1998learning, liu2013bayesian,
  liu2014learning}. Important applications include image
processing~\citep{mignotte2000sonar}, finance~\citep{Barber2015ROCKET}
and neuroscience~\citep{zhu2017graphical}, among others. A major
challenge in real world applications is that the underlying
conditional independence structure can be distorted by
confounders. Unfortunately, despite the large literature on graphical
model estimation, there is limited work to date on estimation with
observed confounding.
 %
%

The observed confounding issue is ubiquitous. Consider the problem of estimating brain functional connectivity from functional magnetic resonance imaging (fMRI)~\citep{biswal1995functional, fox2007spontaneous, shine2015estimation, shine2016temporal} data. Here, the connectivity estimate is known to be susceptible to confounding from physiological noise such as subject motion~\citep{van2012influence, goto2016head}. We emphasize that although the amount of motion is observed, the resulting confounding can significantly distort the connectivity matrix when estimated using conventional means, leading to incorrect scientific inferences~\citep{laumann2016stability}.
Another example is in social network analysis. The social contagion (the effects caused by the people close
to each other in a social network) are shown to be confounded by the effect of an individual's covariates 
on his or her behavior or other measurable responses~\citep{shalizi2011homophily}. As a result, a 
method which is able to recover the social contagion or the social network structure despite the individual confounding is clearly useful.   

This manuscript is motivated by the question: {\bf {\em is it possible to efficiently estimate sparse conditional independence structure between random variables with known confounders?}} We provide a positive answer for the important case of jointly Gaussian random variable models. Although prior works have studied the issue of \emph{hidden} confounders in generative undirected graphical models \citep{jordan1998learning}, to our knowledge, this manuscript is among the first to develop the methodology to deal with observed confounders. 
We propose a new class of graphical models: the partially linear additive Gaussian graphical models (PLA-GGMs), 
whose parameters capture the underlying relationships of random variables, 
and where these relationships take a partially linear additive form~\citep{hastie2017generalized}.
Further, we parametrize the model using two additive components: the \emph{target} i.e. the non-confounded structure, 
and the \emph{nuisance} structure induced by observed confounders.

Importantly, we do not impose a parametric form on the \emph{nuisance} structure -- only 
requiring smoothness to facilitate nonparametric estimation. This significantly 
improves on prior work which has required strong ad-hoc assumptions like the linear 
assumption~\citep{van2012influence, power2014methods} or the zero-expectation assumption~\citep{lee2015joint, geng2018joint} 
on the nuisance parameter. PLA-GGMs are applicable as long as not all the observed samples are highly 
confounded, so that the proposed procedure can compare the confounded samples with the non-confounded ones 
in order to remove the confounding influence. 

We propose a pseudo-profile likelihood (PPL) estimator for learning
PLA-GMMs, which can be considered as a pseudo-likelihood version
profile likelihood~\citep{fan2005profile}.  By minimizing the
$L_1$-regularized negative log PPL, we derive a $\sqrt{n}$-sparsistent
estimator of the \emph{target} structure under mild assumptions. The
sparsistency of the estimator indicates that the proposed method
recovers the true underlying structure with a high probability
\cite{wainwright2009sharp,kolar2009sparsistent,kolar2011time,ravikumar2010high}. We
also show that the convergence rate of the proposed estimator is
faster than competing methods.

The proposed PLA-GGM can be considered as a generative-model counterpart of partially linear additive discriminative models~\citep{fan2008statistical, cheng2014efficient, chouldechova2015generalized,lou2016sparse}. 
Compared with these discriminative models, PLA-GGM as a generative model focuses on estimating the 
relationships among random variables, thus can be used to recover the conditional independence structure,
which discriminative models like \citealt{sohn2012joint, wytock2013sparse} cannot. 
Recall that GGMs can be estimated as a collection of related regressions~\citep{meinshausen2006high}. Along similar linesm,
our PLA-GGM approach requires studying multiple dependent discriminative models simultaneously.


\paragraph{Main Contributions}
Our main technical contributions are summarized as follows:
  \begin{itemize}[leftmargin=*]
  	\vspace{-2mm}
\item To the best of our knowledge, PLA-GMM is the first model to specifically deal with the observed 
   confounders in generative undirected graphical models. Without assuming a parametric form for the 
   confounding, PLA-GMM can accommodate a broad class of potential structure confounders.   
\vspace{-2mm}
\item We demonstrate that PLA-GGMs facilitate $\sqrt{n}$-sparsistent estimators by proposing the PPL method as a new objective for the parameter estimation. Further, since the corresponding minimization problem is shown to be equivalent to a regularized weighted least square, the optimization is shown to be efficient by leveraging the coordinate descent method~\citep{glmnet} and the corresponding strong screening rule~\citep{tibshirani2012strong}.
\end{itemize}
We demonstrate the utility of PLA-GGMs using both synthetic data and  the 1000 Functional Connectomes Project Cobre dataset~\citep{COBRE13:online}, a brain imaging dataset from the Center for Biomedical Research Excellence. The proposed PLA-GGM demonstrates superior accuracy in terms of structure recovery and can effectively detect the abnormalities of the brain functional connectivity of subjects with schizophrenia.  

\section{Modeling}
\label{sec:model}

We begin by formulating PLA-GGMs. For a continuous random vector $\bZ$ and a confounder variable $G$, we assume that 
the conditional distribution $\bZ\given G=g$ follows a Gaussian graphical model~\citep{yang2015graphical} 
with a parameter matrix, denoted by $\bOmega(g)$, that depends on $g$. 
In particular, the conditional distribution of $\bZ\given G=g$ follows:
\begin{align*}
\begin{split}
\text{P}(\bZ = \bz;\bOmega(g)\given G=g)
&\propto\exp\left\{
\sum_{j=1}^p \Omega_{jj}(g) z_j
\right.
\\+&\left. \sum_{j=1}^p \sum_{j'>j}^p \Omega_{jj'}(g) z_jz_{j'}
-\frac{1}{2}\sum_{j}^p z_j^2
\right\},
\end{split}
\end{align*}
where we assume that the diagonal of the covariance matrix of $\bZ\given G=g$ is $1$, without loss of generality \citep{yang2015graphical}. Note that the parameter $\bOmega(g)$ captures the conditional independence structure\footnote{We will refer to {\em conditional independence structure} as {\em structure} for the ease of presentation.} of $\bZ\given G=g$. Therefore, 
the structure of $\bZ$ is allowed to vary based on the values of confounders. 
This characteristic makes the proposed PLA-GGM more general in scope and applicability 
compared to prior work where the structure of $\bZ$ is assumed to be \emph{unrelated} 
to the confounders (see discussion in Section~\ref{sec:competing-method}). 

Let $\bOmega_0 := \bOmega(0)$ represent the non-confounded structure. We assume that the parameter $\bOmega(g)$
takes the partially linear additive form:
\begin{equation}
\label{eq:model:precision}
\bOmega(g) := \bOmega_0 + \bR(g).
\end{equation}
Our goal is to recover $\bOmega_0$ given $n$ independent observations $\mathbb{Z} = \left\{ \bz_i, g_i \right\}_{i \in [n]}$ 
from the joint distribution of $(\bZ, G)$. The term $\bR(g)$ is a nuisance component that arises due to confounding. 
Thus, while the structure of $\bZ$ varies over observations, 
we are only interested in a specific one, $\bOmega_0$, whose sparsity pattern encodes the target non-confounded structure. 

It is clear that recovering $\bOmega_0$ is impossible without constraints on $\bR(\cdot)$. 
For instance, estimates $\left\{\hat{\bOmega}_0 := \bOmega_0/2,\,\hat{\bR}(\cdot) := \bR(\cdot)+\bOmega_0/2  \right\}$ 
and $\left\{\hat{\bOmega}_0 := \bOmega_0/3,\,\hat{\bR}(\cdot) := \bR(\cdot)+2\bOmega_0/3  \right\}$ 
result in the same likelihood, making it impossible to determine the true value of $\bOmega_0$. 
To this end, we enforce a mild assumption -- that the effect of confounders is trivial when the size of confounding itself is small. 
Specifically, for a known $g^*>0$, we assume 
\begin{equation}
\label{eq:asm-pla}
\bR(g) = \bm{0},
\end{equation}
for any $g$ satisfying $\abs{g} \leq g^*$. 

The assumption states that the confounders with values smaller than $g^*$ do not have any effect on the structure of $\bZ$, 
and thus this serves as a constraint on $\bR$. Then, as long as we can observe some samples with the confounders small enough (smaller than $g^*$), 
we should be
able to distinguish $\bOmega_0$ from $\bR(g)$. Such an assumption is much weaker than those used in existing works
including~\citet{van2012influence}, \citet{power2014methods}, and \citet{lee2015joint}, 
where $\bR(g) = \bm{0}$ or $\EE\left[\bR (g) \right]= \bm{0}$ are often assumed. Note that when $g^* = \infty$, \eqref{eq:asm-pla} will 
degenerate to $\bR(g) = 0$. Also, as $g$ is smooth, there should exist infinite $g^*$'s satisfying the definition. We do not require using the largest possible one. The selection of $g^*$ in practice is discussed in Section~\ref{sec:real-world}. 

\section{Pseudo-Profile Likelihood Method}
\label{sec:estimation}

PLA-GGMs facilitate fast-converging estimators.
In this section, we propose an estimation procedure for $\bOmega_0$ in PLA-GGMs. 
\subsection{Pseudo Likelihood}
\label{sec:pl}
For a PLA-GGM parameterized by $\left\{\bR(\cdot), \bOmega_0\right\}$ with observations $\left\{ \bz_i, g_i \right\}_{i \in [n]}$, we first derive the log pseudo likelihood as a linear regression in Lemma~\ref{lem:pl}.
\begin{lemma}
\label{lem:pl}
Define $\bz_{i,-j}$ as the vector $\bz_i$ with the $j^{\text{th}}$ component replaced by $1$, $\bOmega_{0\cdot j}$ the $j^{\text{th}}$ column vector of $\bOmega_{0}$, and  $\bOmega_{i\cdot j}$ the $j^{\text{th}}$ column vector of $\bOmega(g_i)$. The log pseudo likelihood of the PLA-GGM follows
\begin{align}
\label{eq:pl}
\begin{split}
\ell_{PL}& \left(\left\{ \bz_i, g_i \right\}_{i \in [n]};\bR(\cdot), \bOmega_0 \right)
\\:=& \sum_{i=1}^n \sum_{j = 1}^p \left\{ z_{ij} \left( \bz_{i,-j}^\top \bOmega_{0 \cdot j}
+ \bz_{i,-j}^\top \bOmega_{i \cdot j}
\right)
-\frac{1}{2}z_{ij}^2  \right.
\\&-\left.\frac{1}{2}\left( \bz_{i,-j}^\top \bOmega_{0 \cdot j}
+ \bz_{i,-j}^\top \bOmega_{i \cdot j}
\right)^2
\right\}.
\end{split}
\end{align}
\end{lemma}
It should be noticed that \eqref{eq:pl} has the same form as the objective function of $p$ linear regressions each with $n$ observations and $2p$ covariates. Specifically, for the $j^{\text{th}}$ regression ($j \in [p]$), the $n \times 2p$ covariate matrix is defined as 
\begin{equation*}
\begin{bmatrix}
\bx_j & \bx_j
\end{bmatrix}
:= 
\begin{bmatrix}
\bz_{1,-j}^\top & \bz_{1,-j}^\top
\\\bz_{2,-j}^\top & \bz_{2,-j}^\top
\\ \vdots&\vdots
\\\bz_{n,-j}^\top & \bz_{n,-j}^\top
\end{bmatrix},
\end{equation*}
and the corresponding response is $\by_j := \left[z_{1j}, z_{2j}, \cdots, z_{n,j}  \right]^\top$.  

For graphical models without confounders it is known that minimizing $L_1$-regularized negative 
log PL~\citep{geng2017efficient,geng2018stochastic, geng2018temporal,  kuang2017screening} can lead to $\sqrt{n}$-sparsistent parameter estimators~\citep{yang2015graphical}. Unfortunately, 
this is no longer true for PLA-GGMs, since the number of unknown nuisance parameters, which are non-parametric,
is far too large. Instead, we leverage kernel methods and propose an approximate PL.

\subsection{Pseudo Profile Likelihood}
\label{sec:ppl}

We propose a new inductive principle to estimate $\bOmega_0$. As mentioned in Section~\ref{sec:pl}, 
the varying confounding $\bR(g_i)$'s are an obstruction to estimating $\bOmega_0$. 
We summarize the varying effects as $M_{ij} := \bx_{ij}^\top \bOmega_{i \cdot j}$, 
where $\bx_{ij}^\top$ denotes the $i^{th}$ row vector of $\bx_j$. \eqref{eq:pl} is transformed to 
\begin{align*}
\begin{split}
\ell_{PL} &\left(\left\{ \bz_i, g_i \right\}_{i \in [n]};\bR(\cdot), \bOmega_0  \right)
\\=& \sum_{i=1}^n \sum_{j = 1}^p \left\{ z_{ij} \left(  \bx_{ij}^\top \bOmega_{0 \cdot j}
+ M_{ij}
\right)
-\frac{1}{2}z_{ij}^2  \right.
\\&-\left.\frac{1}{2}\left(  \bx_{ij}^\top \bOmega_{0 \cdot j}
+ M_{ij}
\right)^2
\right\}.
\end{split}
\end{align*}
There are two unknown parts in PL: $\bOmega_0$ and $M_{ij}$. Intuitively, if we can
express $M_{ij}$'s using $\bOmega_0$, we will be able to omit $M_{ij}$ and focus on 
estimating $\bOmega_0$. This leads to the following  Lemma on approximating $M_{ij}$'s.

\begin{lemma}
\label{lem:m}
 For the $i^{\text{th}}$ observation, we define an $n\times n$ kernel weight matrix , $\bW_i$, which is a diagonal matrix with $\left[\psi\left(\abs{g_i-g_1}/h\right), \psi\left(\abs{g_i-g_2}/h\right), \cdots, \psi\left(\abs{g_i-g_n}/h\right)   \right]$. $\psi(\cdot)$ is a symmetric kernel density function, and $h > 0$ is a user specified bandwidth. Then, we define an auxiliary matrix:
\begin{equation*}
\bD_{ij} := 
\begin{bmatrix}
\mathds{1}_{\left\{ \abs{g_1}\geq g^* \right\}}\bz_{1,-j}^\top&\frac{g_1 - g_i}{h}\mathds{1}_{\left\{ \abs{g_1}\geq g^* \right\}}\bz_{1,-j}^\top
\\ \mathds{1}_{\left\{ \abs{g_2}\geq g^* \right\}}\bz_{2,-j}^\top&\frac{g_2 - g_i}{h}\mathds{1}_{\left\{ \abs{g_2}\geq g^* \right\}}\bz_{2,-j}^\top
\\ \vdots & \vdots
\\ \mathds{1}_{\left\{ \abs{g_n}\geq g^* \right\}}\bz_{n,-j}^\top&\frac{g_n - g_i}{h}\mathds{1}_{\left\{ \abs{g_n}\geq g^* \right\}}\bz_{n,-j}^\top
\end{bmatrix},
\end{equation*}
where 
\begin{equation*}
\mathds{1}_{\left\{ \abs{g}\geq g^* \right\}} := 
\begin{cases}
1 & if \abs{g}\geq g^*
\\ \leq 1 & if \abs{g}< g^*
\end{cases},
\end{equation*} 
satisfying the smoothing assumptions in Section~\ref{sec:assumption}. 

An estimator of $M_{ij}$ can be derived as
$\hat{M}_{ij}:= \bS_{ij}^\top \left(\by_j - {\bx}_j\bOmega_{0\cdot j}\right)$, where
\begin{equation*}
\bS_{ij}^\top := \begin{bmatrix}
\bx_{ij}^\top & 0 
\end{bmatrix}
\left(
\bD_{ij}^\top \bW_i \bD_{ij}
\right)^{-1}
\bD_{ij}^\top \bW_i.
\end{equation*}
\end{lemma}

The function $\mathds{1}_{\left\{ \abs{g}\geq g^* \right\}}$ in Lemma~\ref{lem:m} is a user-specified function. In Theorem~\ref{thm:sparsistency}, we show that the value of $\mathds{1}_{\left\{ \abs{g}\geq g^* \right\}}$ does not affect the $\sqrt{n}$-sparsistency of the estimation, as long as it satisfies the definitions in Lemma~\ref{lem:m}. 

Note that given the observations, $\hat{M}_{ij}$ is only dependent on $\bOmega_0$. Therefore, by replacing $M_{ij}$ with $\hat{M}_{ij}$ in \eqref{eq:pl} and some additional transformations, we can derive an approximate log pseudo likelihood whose only unknown parameter is $\bOmega_0$. We define this as the log pseudo profile likelihood (PPL):

\begin{definition}[PPL]
	\label{def:ppl}
Following the notations above, the log PPL is defined as 
\begin{align}
\label{eq:ppl}
\begin{split}
&\ell_{PPL} \left(\left\{ \bz_i, g_i \right\}_{i \in [n]};\bR(\cdot), \bOmega_0  \right)
\\:=&
\ell_{PPL} \left(\left\{ \bz_i, g_i \right\}_{i \in [n]};\bOmega_0 \right)
\\:=& \sum_{i=1}^n \sum_{j = 1}^p \left\{ \vphantom{\int_1^2}  \left( \bm{1}_{i} - \bS_{ij}\right)^\top\by_j \left( \bm{1}_{i} - \bS_{ij}\right)^\top \bx_j \bOmega_{0\cdot j}\right.
\\&-\left.
\frac{1}{2}\left[\left( \bm{1}_{i} - \bS_{ij}\right)^\top\by_j\right]^2 \right.
\\ &-\left.\frac{1}{2}\left[\left( \bm{1}_{i} - \bS_{ij}\right)^\top \bx_j \bOmega_{0\cdot j}
\right]^2
\right\},
\end{split}
\end{align}
where $\bm{1}_{i}$ is an $n\times 1$ vector, whose $i^\text{th}$ component is $1$ and others are $0$'s.
\end{definition}

The proposed PPL shares a close relationship with the profile likelihood~\citep{speckman1988kernel, fan2005profile}: if the components of $\bZ_i$ are independent of each other, the form of PPL is equivalent to the log profile likelihood. However, we do not make any assumptions on the independence here, which makes PPL a type of log pseudo likelihood. Such a rationale of intentionally overlooking the dependency is widely used in the derivation of various types of pseudo likelihoods including the one in \citealt{huang2012maximum}. However, \citealt{huang2012maximum} focus on Cox regression for the longitudinal data analysis which is different from our setting. Also, the inductive principle in \citealt{huang2012maximum} emphasizes the consistency, while we will show that a $\sqrt{n}$-sparsistent estimator can be achieved by using the PPL.

\subsection{$L_1$-Regularized MaPPLE}
\label{sec:mpple}
With the proposed PPL \eqref{eq:ppl}, we can now derive an estimator for $\bOmega_0$. For the ease of presentation, we will use $F(\bOmega_0)$ to denote $\frac{-\ell_{PPL} \left(\left\{ \bz_i, g_i \right\}_{i \in [n]};\bR(\cdot), \bOmega_0  \right)}{n}$. Then, the $L_1$-regularized MaPPLE is derived as
\begin{align}
\label{eq:l1-mpple}
\begin{split}
\hat{\bOmega}_0 := \argmin_{\bOmega_0} F(\bOmega_0)+ \lambda \norm{\bOmega_0},
\end{split}
\end{align}
where $\norm{\bOmega_0} = \sum_{j}^p \sum_{j'>j}^p \abs{\Omega_{0jj'}} $, and $\lambda$ is the regularization parameter.

Note that \eqref{eq:l1-mpple} has the same form as a regularized weighted least square problem. Therefore, the optimization can be efficiently solved using the coordinate descent method~\citep{glmnet}, combined with the strong screening rule~\citep{tibshirani2012strong}. We implement the optimization using the R package glmnet~\citep{glmnet}.

\section{Sparsistency of the $L_1$-Regularized MaPPLE}
\label{sec:consistency}
The $L_1$-regularized MaPPLE \eqref{eq:l1-mpple} is proved to be $\sqrt{n}$-sparsistent under some mild assumptions.

\subsection{Assumptions}
\label{sec:assumption}
To start with, we discuss the assumptions for the estimator. Since the estimation of $M_{ij}$ in Lemma~\ref{lem:m} is based on kernel methods, we need some standard assumptions widely used in this literature~\citep{mack1982weak, fan2005profile, kolar2010estimating}. The following assumptions are concerned with the order of $n$, $p$, and $h$, and the smoothness.
\begin{assumption}
\label{asm:order-n-h}
Define $c_n = \sqrt{\frac{-\log h}{nh}} + h^2$ with $h \in (0,1)$ and $p>1$. Then, we assume that there exists $C_1>0$, so that $c_n^2 \leq C_1 \sqrt{\frac{\log p}{n}}$.
\end{assumption}
\begin{assumption}
	\label{asm:smooth}
	For any $g$, the following matrices are all element-wise Lipschitz continuous with respect to $g$:
	\begin{align*} 
&\EE\left(\bZ^\top \bZ \given G = g  \right), 
\\ &\EE\left(\mathds{1}^2_{\left\{ \abs{g}\geq g^* \right\}}  \bZ^\top \bZ \given G = g\right), 
\\&\text{and} \,\, \EE\left(\mathds{1}^2_{\left\{ \abs{g}\geq g^* \right\}}  \bZ^\top \bZ \given G = g\right)^{-1}.
\end{align*}

\end{assumption}
Also, since we do not pose parametric assumptions to $\bR(g)$ and $f(\cdot)$, we further need the following assumptions on both. 
\begin{assumption}
	\label{asm:f-r}
	The random variable $G$ has a bounded support, and $f(\cdot)$ is Lipschitz continuous and bounded away from 0 on its support. $\bR(g)$ has continuous second derivative. 
	\end{assumption}

Next, we introduce an assumption required for sparsistency. The following mutual incoherence condition is vital to the sparsistency~\citep{ravikumar2010high}. Here, we define $\bOmega_0^*$ as the underlying parameter, and treat $\bOmega^*_0$ as a vector containing all the components without repeats.

\begin{assumption}
	\label{asm:incoherence} 
	Define $A$ as the index set of the non-diagonal and non-zero components of $\bOmega_0^*$, $D$ as the index set of the diagonal components of $\bOmega_0^*$, and $N$ as the index set of the non-diagonal and zero components of $\bOmega_0^*$. Define the incoherence coefficient as $0<\alpha<1$. Then for $\bH=\grad^2 F(\bOmega_0^*)$, there exists $C_2 > 0$, so that $\Norm{\bH_{NS}\bH_{SS}^{-1}}_{\infty} \le 1-\alpha$ and $\Norm{\bH_{SS}^{-1}}_{\infty}\le C_2$, where we use the index sets as subscripts to represent the corresponding components of a vector or a matrix.
\end{assumption}
Our final assumption is required by the fixed point proof technique we apply \citep{ortega2000iterative, yang2011use}, and may not be necessary for more calibrated proofs. 
\begin{assumption}
	\label{asm:fix}
	Define $\bm{R}(\bDelta):=\grad F(\bOmega_0)-\grad F(\bOmega_0^*) - \grad^2 F(\bOmega_0^*) (\bOmega_0-\bOmega_0^*)$, where $\norm{\bDelta}_{\infty} \leq r:= 4C_2 \lambda \leq \frac{1}{C_2 C_3}$ with $\bDelta_N = \bzero$, and for some $C_3>0$. Then $\Norm{\bm{R}(\bDelta)}_{\infty} \le C_3 \Norm{\bDelta}_{\infty}^2$.
\end{assumption}

\subsection{Main Theoretical Results}
With the assumptions in Section~\ref{sec:assumption}, the $\sqrt{n}$-sparsistency of the $L_1 $-regularized MaPPLE is provided in Theorem~\ref{thm:sparsistency}.

\begin{theorem}
	\label{thm:sparsistency}
	Suppose that Assumption~\ref{asm:order-n-h} - \ref{asm:fix} are satisfied.  Then, for any $\epsilon >0$, with probability of at least $1-\epsilon$, there exists $C_4>0$, so that $\hat{\bOmega}_0$ shares the same structure with the underlying true parameter $\bOmega_0^*$, if for some constant $C_5>0$,
	\begin{align*}
	&C_5 \sqrt{\frac{\log p}{n}} \geq \lambda  \geq  \frac{4}{\alpha} C_4\sqrt{\frac{\log p}{n}},
	\\&r :=4C_2 \lambda \le \Norm{\bOmega^*_{0S}}_{\infty},
	\end{align*}
	and $n\ge \left(64C_5  C_2^2C_3/\alpha\right)^2\log p$.
\end{theorem}

According to Theorem~\ref{thm:sparsistency}, the $L_1$-regularized MaPPLE recovers the true structure of $\bOmega_0$with a high probability. Also, the scale of the estimation error denoted by $r$ is less than $4C_2C_5 \sqrt{\frac{\log p}{n}}$, which converges to zero at a rate of $\sqrt{n}$. In other words, the smallest scale of the non-zero component that the PPL method can distinguish from zero in the true parameter converges to zero at a rate of $\sqrt{n}$. We refer to this result as $\sqrt{n}$-sparsistency.

Such a convergence rate is faster than ordinary nonparametric methods, which often have a $n^{-2/5}$ convergence rate~\citep{speckman1988kernel,kolar2010estimating}. Also, the $\sqrt{n}$-sparsistency matches the results of semi-parametric methods~\citep{fan2005profile,fan2008statistical} for discriminative models, where the estimated parametric part is shown to be $\sqrt{n}$-consistent.  

\begin{figure*}[t]
	\centering
	\begin{subfigure}{0.23\textwidth}
		\centering
		\includegraphics[scale=0.2]{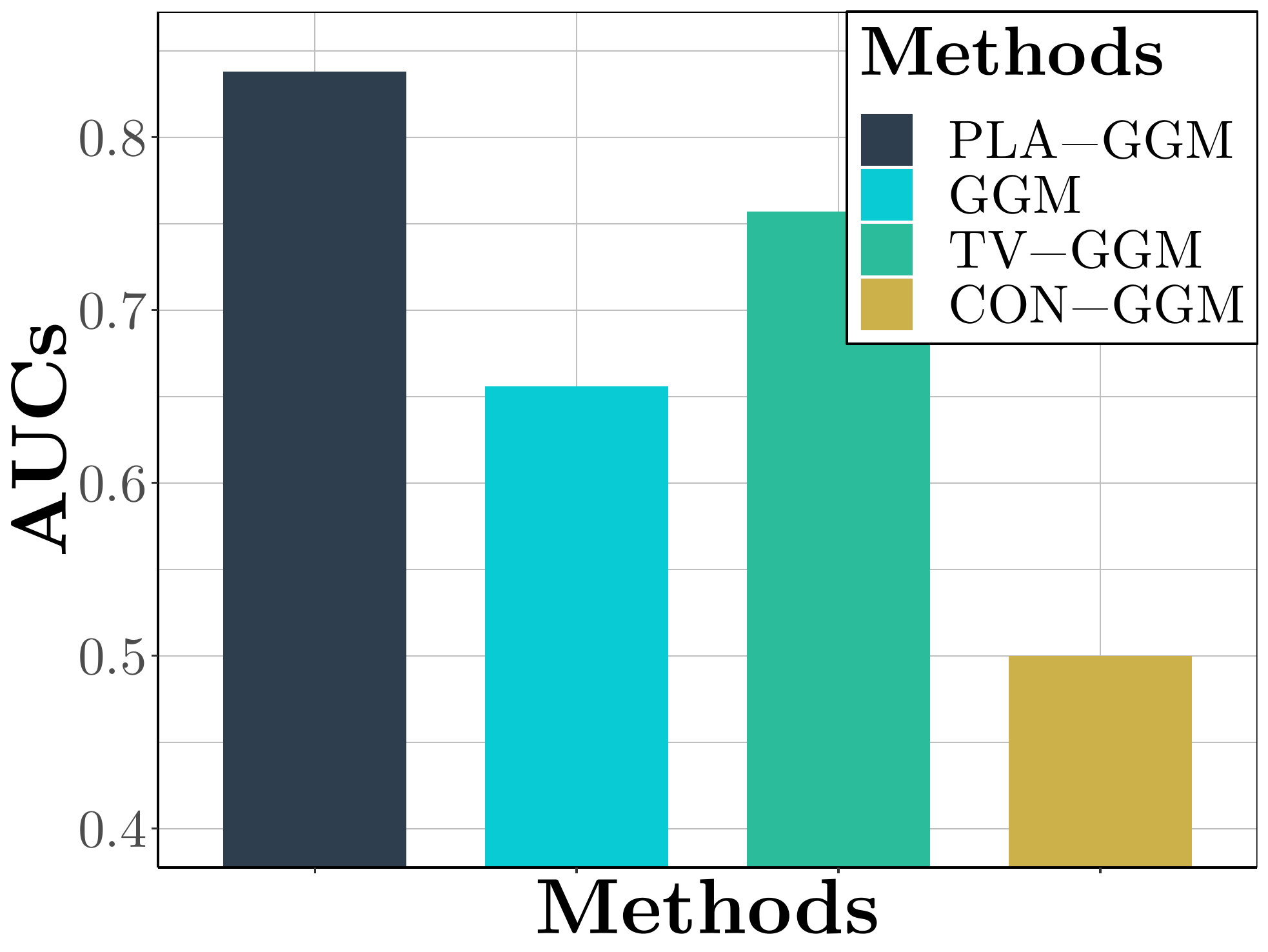}
		\caption{$p = 10$}
	\end{subfigure}
	\centering
	\begin{subfigure}{0.23\textwidth}
		\centering
		\includegraphics[scale=0.2]{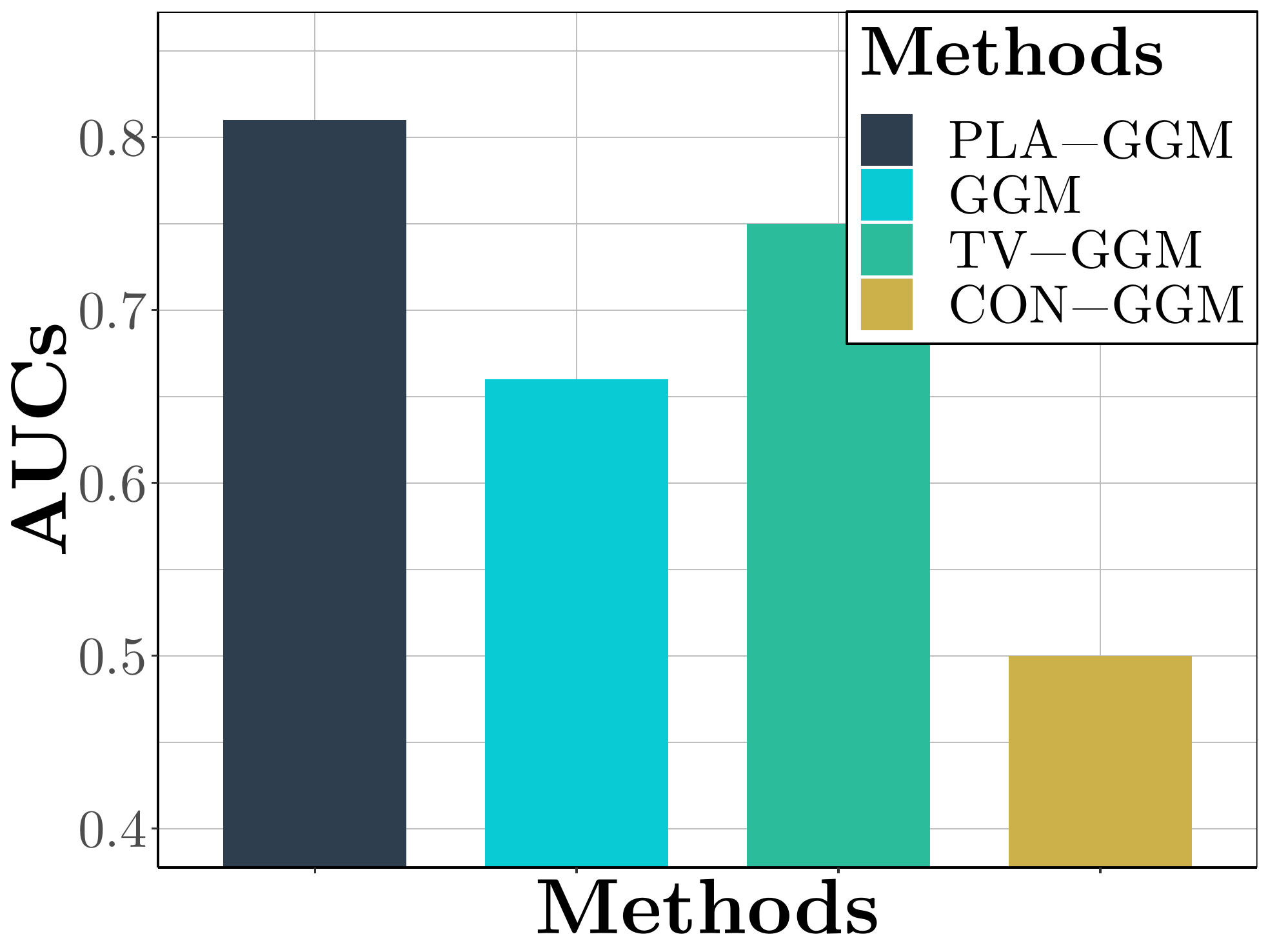}
		\caption{$p = 20$}
	\end{subfigure}
	\centering
	\begin{subfigure}{0.23\textwidth}
		\centering
		\includegraphics[scale=0.2]{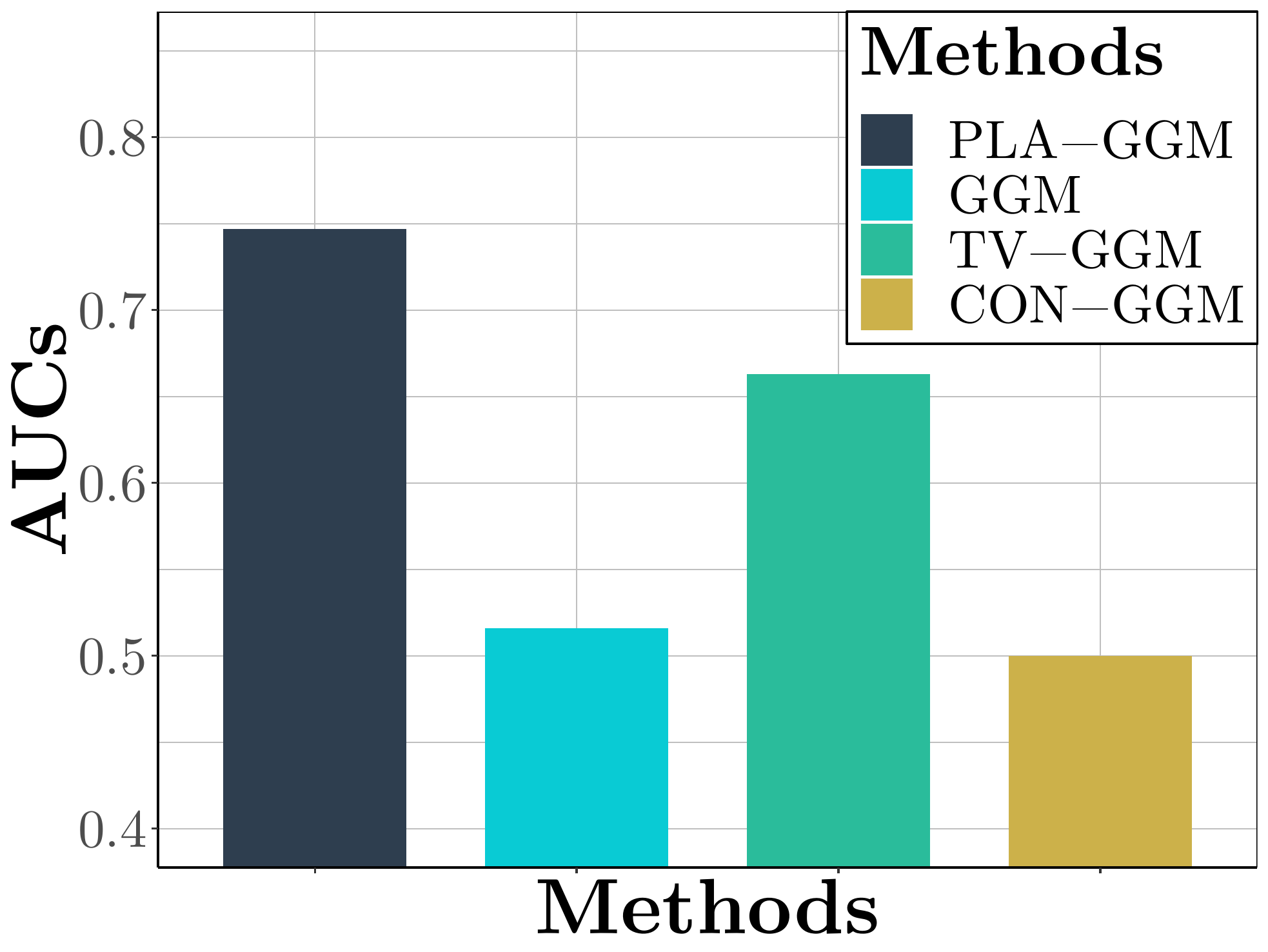}
		\caption{$p = 50$}
	\end{subfigure}
	\centering
	\begin{subfigure}{0.23\textwidth}
		\centering
		\includegraphics[scale=0.2]{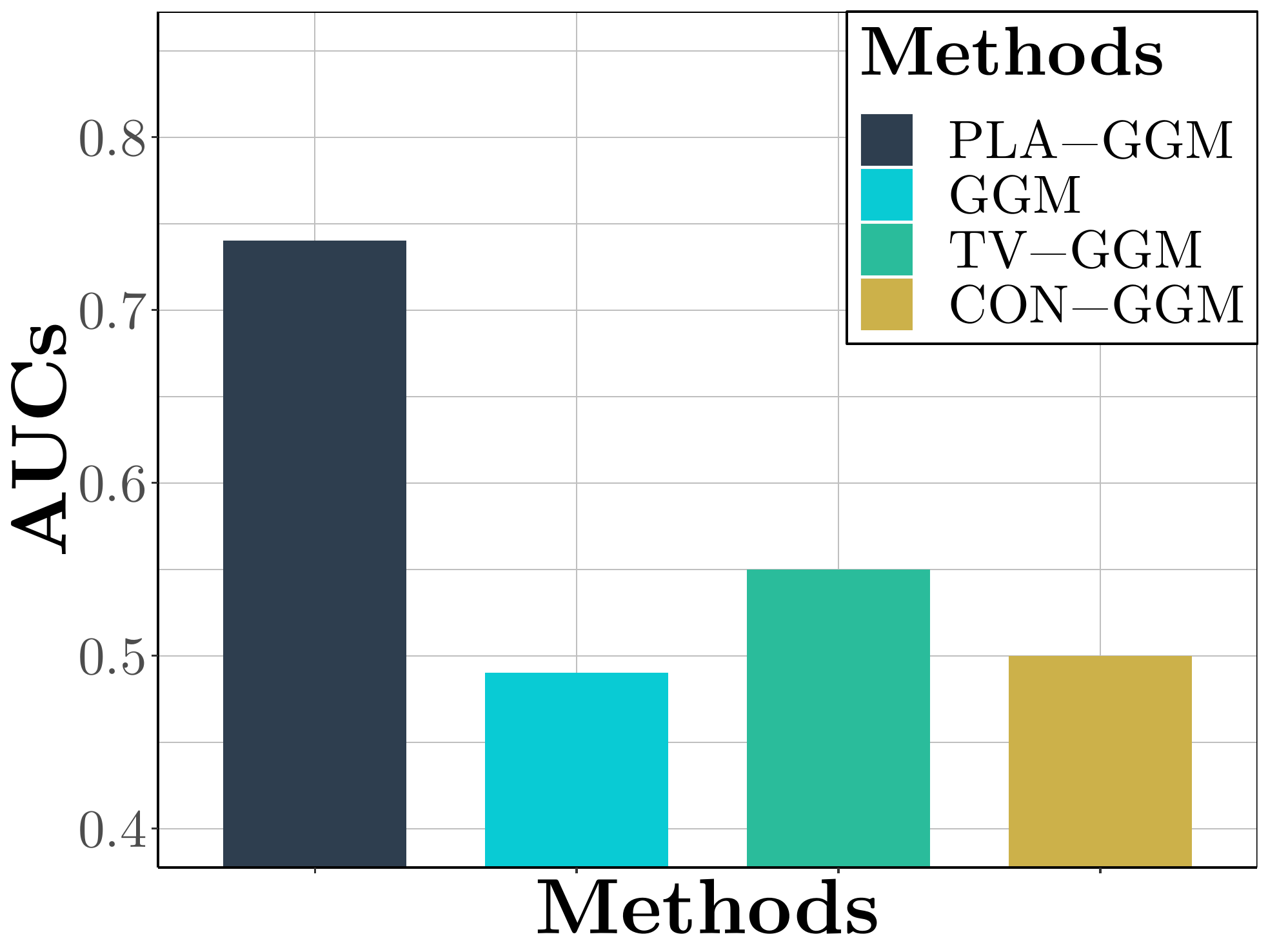}
		\caption{$p = 100$}
	\end{subfigure}
	\centering
	\caption{Area under curve (AUC) of considered methods for the structure learning with different numbers of variables.}
	\label{fig:auc}
\end{figure*}

In Theorem~\ref{thm:sparsistency}, the value of $g^*$ does not affect the $\sqrt{n}$-sparsistency of the estimator. In practice, however, if $g^*$ is too small, the $\left(
\bD_{ij}^\top \bW_i \bD_{ij}
\right)$ tends to be singular, since few samples are observed with $\abs{g} \leq g^*$. Accordingly, the PPL method will be not applicable. Therefore, we need to observe some non-confounded samples to implement the PPL method. The $\sqrt{n}$-sparsistency is not directly related to the selected $\mathds{1}_{\left\{ \abs{g}\geq g^* \right\}}$ either. Along the proof of Theorem~\ref{thm:sparsistency} in the Supplements, we notice that $\mathds{1}_{\left\{ \abs{g}\geq g^* \right\}}$ (and thus $g^*$) can only affect some auxiliary constants. Since this relationship is neither significant, nor straightforward, we do not discuss it here.

\section{Related Methods}
\label{sec:competing-method}

After a thorough analysis on the proposed PLA-GGM and PPL method, we now study some related methods that fall into four categories: Gaussian graphical models incorporating confounders, denoted by CON-GGMs; the Gaussian graphical models using linear regression to deal with confounders~\citep{van2012influence, power2014methods} denoted by LR-GGMs; original Gaussian graphical models only using non-confounded samples, denoted by GGMs; and time-varying Gaussian graphical models~\citep{kolar2010estimating, yang2015fused} denoted by TV-GGMs. Theoretically, the proposed PLA-GGM is more generalized and facilitates faster-converging estimators than the existing models.   

\subsection{CON-GGMs and LR-GGMs}

Although not designed for this task, it is possible to apply more standard graphical modeling approaches 
to deal with some of the effects of observed confounders. For instance, a straightforward 
alternative to PLA-GGMs is to directly incorporate the confounder as a random variable into the GGM. 
Specifically, CON-GGMs assume that the confounder $G$ follows a GGM jointly with the random vector $\bZ$, 
which means
\begin{equation}
\label{eq:con-ggm}
\left( G, \bZ \right) \sim \text{GGM} (\bOmega),
\end{equation} 
where the joint covariance matrix follows
\begin{equation*}
\bSigma :=\bOmega^{-1}=
\begin{bmatrix}
\bSigma_{\bZ \bZ}&\bSigma_{\bZ G}
\\ \bSigma_{G \bZ} & \bSigma_{GG}
\end{bmatrix}.
\end{equation*}
Since the target structure is for $\bZ\given G=0$, we can estimate $\bOmega$ by graphical Lasso~\citep{friedman2008sparse} first, and then derive the inverse conditional covariance matrix for $\bZ\given G=0$.

LR-GGM is a model widely used in the neuroscience area~\citep{van2012influence, power2014methods}, assuming that the confounders will cause a linear confounding to the observed samples. The model can be formulated as: 
\begin{equation}
\label{eq:lr-ggm}
\bZ = \bbeta^{\top}G + \bZ',
\end{equation} 
where $\bZ'$ follows a Gaussian graphical model with parameter $\bOmega$, and $G$ satisfies Assumption~\ref{asm:f-r}. Since conditional on $G=0$, $\bZ$ is equivalent to $\bZ'$, the target parameter for the non-confounded structure is just $\bOmega$. LR-GGMs use linear regressions to recover $\bbeta$, and further to regress out confoundings. Finally, LR-GGMs estimate $\bOmega_0$ by graphical Lasso.

By deriving the inverse covariance matrices of $\bZ$ conditional on $G$ for both CON-GGMs and LR-GGMs, it should be noticed that the inverse conditional covariance matrices are irrelevant to the value of $G$. In other words, the confounder $G$ does \emph{not} affect the conditional independence structure of $\bZ$, which is often an unrealistic restriction. In contrast, PLA-GGM particularly deals with confounding of the structure by $G$. Further following this direction, we can derive the following theorem which describes the the relationship among CON-GGMs, LR-GGMs, and PLA-GGMs. 
\begin{theorem}
\label{thm:con-lr-ggm}
The CON-GGM \eqref{eq:con-ggm} and the LR-GGM \eqref{eq:lr-ggm} are two special cases of the PLA-GGM by respectively assuming: 

\begin{itemize}
\item $G$ follows a normal distribution, $\bR(g):=0$ and $\bOmega_0 := \left[
\bSigma_{\bZ \bZ}  
- \bSigma_{\bZ G}
\bSigma_{GG}^{-1}
\bSigma_{G \bZ} 
\right]^{-1}$; 
\item $\bR(g) := 0$.
\end{itemize}

\end{theorem}

Thus, it is clear that CON-GGMs and LR-GGMs both assume a \emph{constant} underlying structure irrelevant to $G$, and are parametric special cases of the proposed PLA-GGMs. Also, since the two methods assume $\bR(g)=0$ either exactly or asymptotically, they will treat the average of $\bOmega(g)$ as the underlying $\bOmega_0$ and derive incorrect structures that are too dense.

\subsection{GGMs and TV-GGMs}

In PLA-GGMs, it is assumed that some non-confounded samples are
observed. Therefore, we can directly apply GGM to the non-confounded
samples and estimate the structure. However, by doing this, only the
information from the non-confounded data are used. The estimators are
obviously not $\sqrt{n}$-sparsistent, considering that most of the $n$
observed samples are confounded and not used by GGM. Thus, the method
is less accurate than PLA-GGM.

Another class of relevant methods are time-varying graphical models
(TV-GMs)~\citep{le09keller,song09time,kolar2010estimating,kolar10estimating}, used to
estimate a different parameter at each time point or
observation. Specifically, TV-GGMs assume that $\bZ$ follows a varying
Gaussian graphical model over $G$. Methods like fused
Lasso~\citep{yang2015fused, zhu2018clustered} and the kernel
estimation~\citep{kolar2010estimating} are applied to estimate the
varying structure of $\bZ$. One may consider applying such a model,
then perhaps averaging the time-varying graph to estimate the
non-confounded component. However, since the target of such methods
are multiple structures, the estimators can only be guaranteed to be
$n^{-2/5}$-consistent. In contrast,
PLA-GGMs 
use all the samples to recover the parameter representing the
underlying non-confounded structure, and can achieve
$\sqrt{n}$-sparsistency.

\begin{figure*}[t]
	\centering
	\begin{subfigure}{0.23\textwidth}
		\centering
		\includegraphics[scale=0.6]{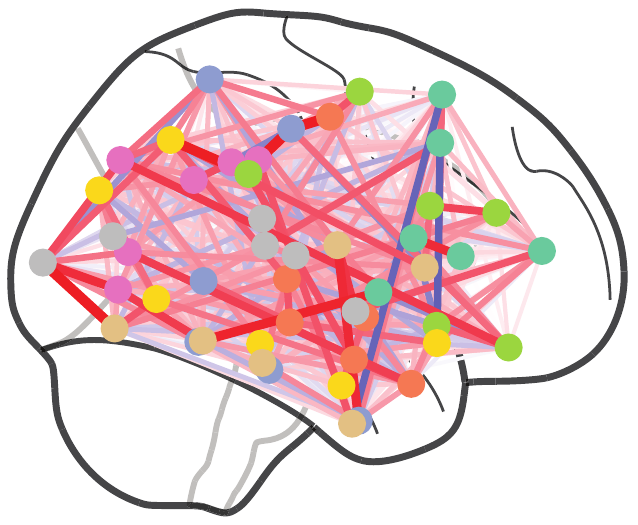}
		\caption{Controls\\PLA-GGM}
	\end{subfigure}
	\centering
	\begin{subfigure}{0.23\textwidth}
		\centering
		\includegraphics[scale=0.6]{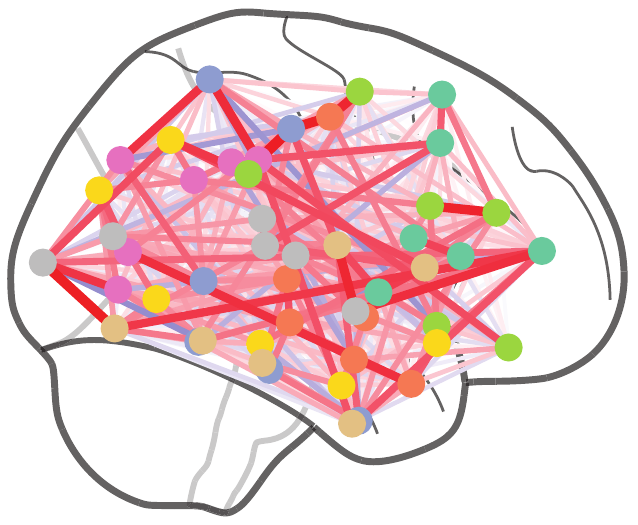}
		\caption{Schizophrenia\\PLA-GGM}
	\end{subfigure}
	\centering
	\begin{subfigure}{0.23\textwidth}
		\centering
		\includegraphics[scale=0.6]{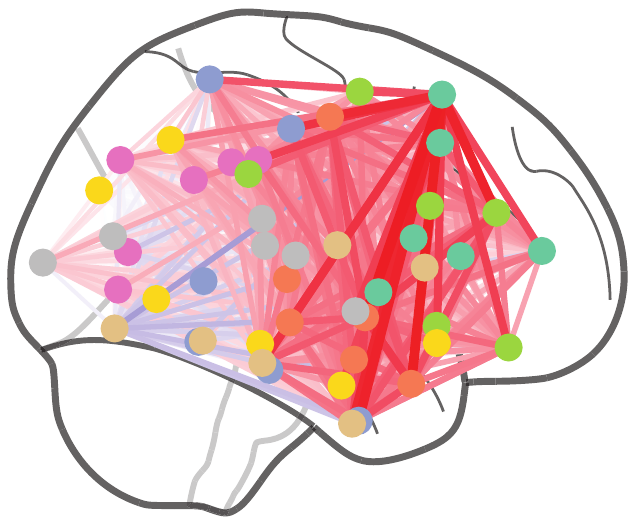}
		\caption{Controls\\LR-GGM}
	\end{subfigure}
	\centering
	\begin{subfigure}{0.23\textwidth}
		\centering
		\includegraphics[scale=0.6]{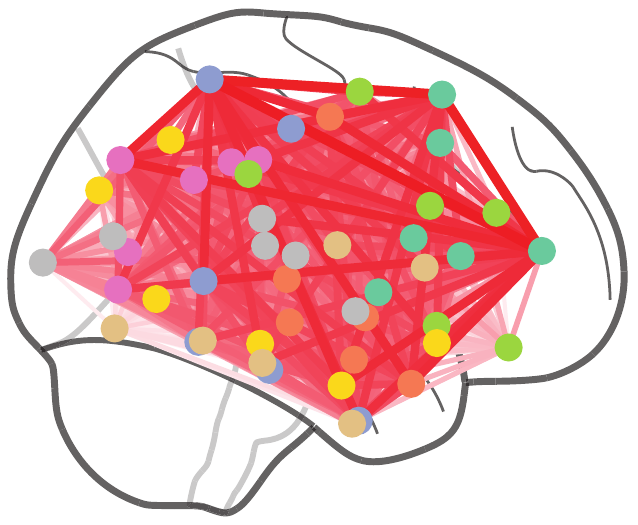}
		\caption{Schizophrenia\\LR-GGM}
	\end{subfigure}
	\centering
	\caption{Glass brains for the estimated brain functional connectivity for subjects with schizophrenia and controls using PLA-GGMs and LR-GGMs. }
	\label{fig:glass-brain}
\end{figure*}

\subsection{Graphical Models with Nonparametric Methods}

PLA-GMM is not the first approach to incorporate nonparametric methods
into graphical models. Prior works like
\citealt{liu2009nonparanormal,kolar10nonparametric,voorman2013graph, Wang2014Inference,
  suggala2017expxorcist}, and \citealt{lu2015kernel, Lu2015Posta} have tried to
relax the parametric definition of graphical models to realize a more
generalized model. However, these methods do not help much to deal
with observed confounders, since the structure among the random
variables is assumed to be independent of the values of the
confounders. Partially linear additive models have also been combined
with directed acyclic graphs in \cite{rothenhausler2018causal}, which
was developed for causal inference and not the structure analysis.

\section{Experiments}
\label{sec:experiment}
To demonstrate the empirical performance of the proposed PLA-GGM and PPL method, we apply them to synthetic data for a structure recovery task in Section~\ref{sec:synthetic} and a real fMRI dataset for a brain functional connectivity estimation task in Section~\ref{sec:real-world}.

\subsection{Structure Recovery}
\label{sec:synthetic}  
In this section, we use simulated data to compare PLA-GGM, TV-GGM and CON-GGM discussed in Section~\ref{sec:competing-method}, with the proposed PLA-GGM for structure recovery. We simulate data from PLA-GGMs following the procedure provided in the Supplement.
We consider the case of $p=10, 20, 50, 100$. For all these settings, we fix $n = 800$ samples. Then, the four methods are applied to the generated datasets to recover the underlying conditional independence structure. The regularization parameter $\lambda$ is selected by 10-fold cross validation from a series of auto generated $\lambda$'s by glmnet. The bandwidth is determined according to Assumption~\ref{asm:order-n-h}. We use 
$\mathds{1}_{\left\{ \abs{g}\geq g^* \right\}} =1- \exp(-k^2g^2)/2$,
where $k$ is selected according to the designated $g^*$. We have also studied other forms for $\mathds{1}_{\left\{ \abs{g}\geq g^* \right\}}$, which did not significantly affect the performance. 

The achieved area under curve (AUC) of the receiver operating characteristic using the selected hyper parameters are reported in Figure~\ref{fig:auc}. Consistent with the analysis in Section~\ref{sec:competing-method}, the proposed PLA-GGM achieves higher AUCs on structure recovery than the competing methods. Also, as the number of variables increases, the advantage of PLA-GGM gets more significant. The phenomenon results from the $\sqrt{n}$-sparsistency of the $L_1$-regularized MaPPLE, which is more accurate and requires less data. It should also be noticed that the AUC achieved by CON-GGM is always around 0.5. The reason is that, following the data simulation procedure, the true $\bOmega(g_i)$'s are always dense, although $\bOmega_0$ is sparse. As suggested by the analysis in Section~\ref{sec:competing-method}, CON-GGM treats $\bOmega(g)$ as the $\bOmega_0$, and thus tends to recover a wrongly dense $\bOmega_0$.    

\subsection{Brain Functional Connectivity Estimation}
\label{sec:real-world}
We apply the PLA-GGM to the 1000
Functional Connectomes Project Cobre dataset~\citep{COBRE13:online}, from the Center for Biomedical Research Excellence. The dataset contains 147 subjects with 72 subjects with schizophrenia and 75 healthy controls. For each subject, resting state fMRI time series and the corresponding confounders are recorded. We use the 7 confounders
provided in the dataset relate to motion for the analysis, and apply Harvard-Oxford Atlas to
select the 48 atlas regions of interest (ROIs). Additional preprocessing details are deferred to the dataset authors~\citep{COBRE13:online}. The performance of PLA-GGM is compared to LR-GGM, which is the most widely-used method to deal with motion confounding in the fMRI literature~\citep{van2012influence, power2014methods}. 

We use 
$\mathds{1}_{\left\{ \abs{g}\geq g^* \right\}} =1- \exp(-100g^2)/2$ for the following analysis, which is equivalent to $g^*=0.578$. If the selected $g^*$ is less than the largest possible value, the estimation should still be accurate, since \eqref{eq:asm-pla} is satisfied. However, a too small $g^*$ may induce a singular $\left(
\bD_{ij}^\top \bW_i \bD_{ij}
\right)$ and thus the failure of the PPL method. We select the smallest $g^*$ where PPL can be successfully implemented, and use the corresponding $\mathds{1}_{\left\{ \abs{g}\geq g^* \right\}}$. The results using other $g^*$'s and $\mathds{1}_{\left\{ \abs{g}\geq g^* \right\}}$'s are reported in the Supplements. Due to Theorem~\ref{thm:sparsistency}, the form of $\mathds{1}_{\left\{ \abs{g}\geq g^* \right\}}$ does not affect the sparsistency of the estimator, and thus has a limited effect on the performance. 

\begin{figure*}[t]
	\begin{minipage}[b]{0.7\linewidth}
		\begin{subfigure}{0.5\linewidth}
			\centering
			\includegraphics[scale=0.27]{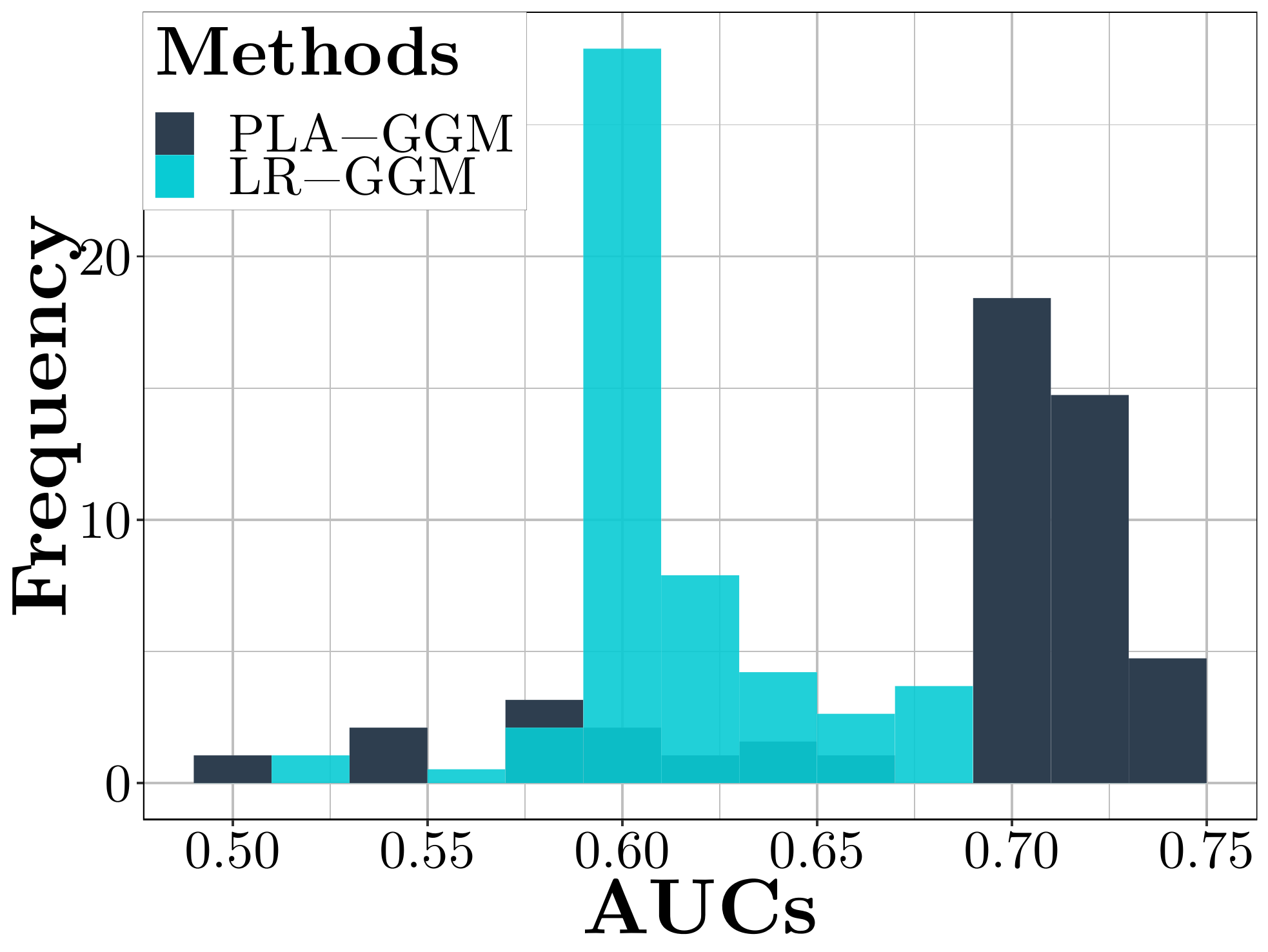}
			\centering
			\caption{Diagnosis using structures.}
			\label{fig:auc-structure-class}
		\end{subfigure}
		\begin{subfigure}{0.5\linewidth}
			\centering
			\includegraphics[scale=0.27]{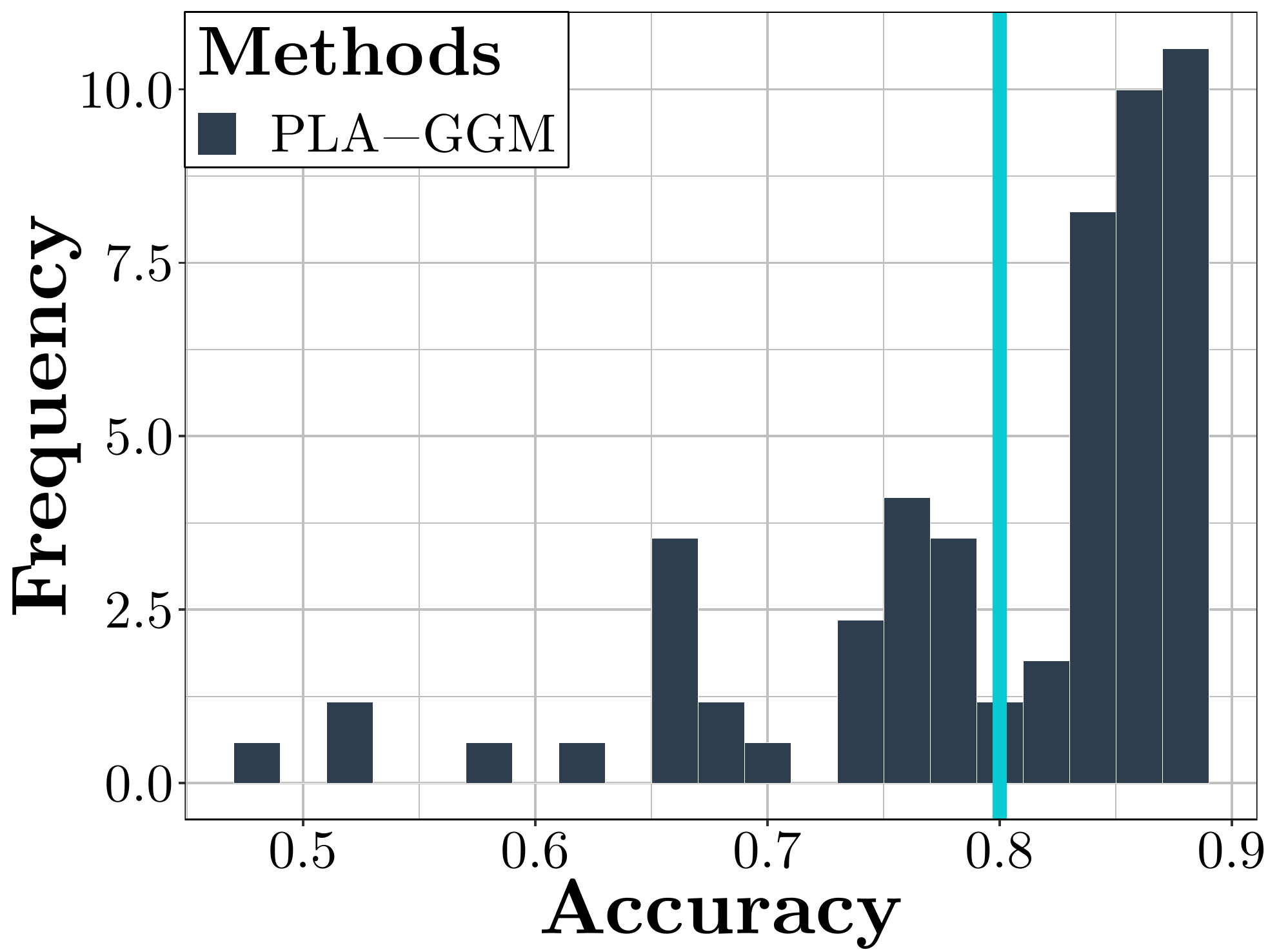}
			\caption{Diagnosis using $\hat{\bOmega}_0$.}
			\label{fig:auc-all}
		\end{subfigure}
		\caption{AUCs for the diagnosis of schizophrenia using only the structures or $\hat{\bOmega}_0$ with different regularization parameters.}
	\end{minipage}
	\vrule
	\hspace{0.2mm}
	\begin{minipage}[t][2mm][b]{0.28\linewidth}
		\centering
		\includegraphics[scale=0.27]{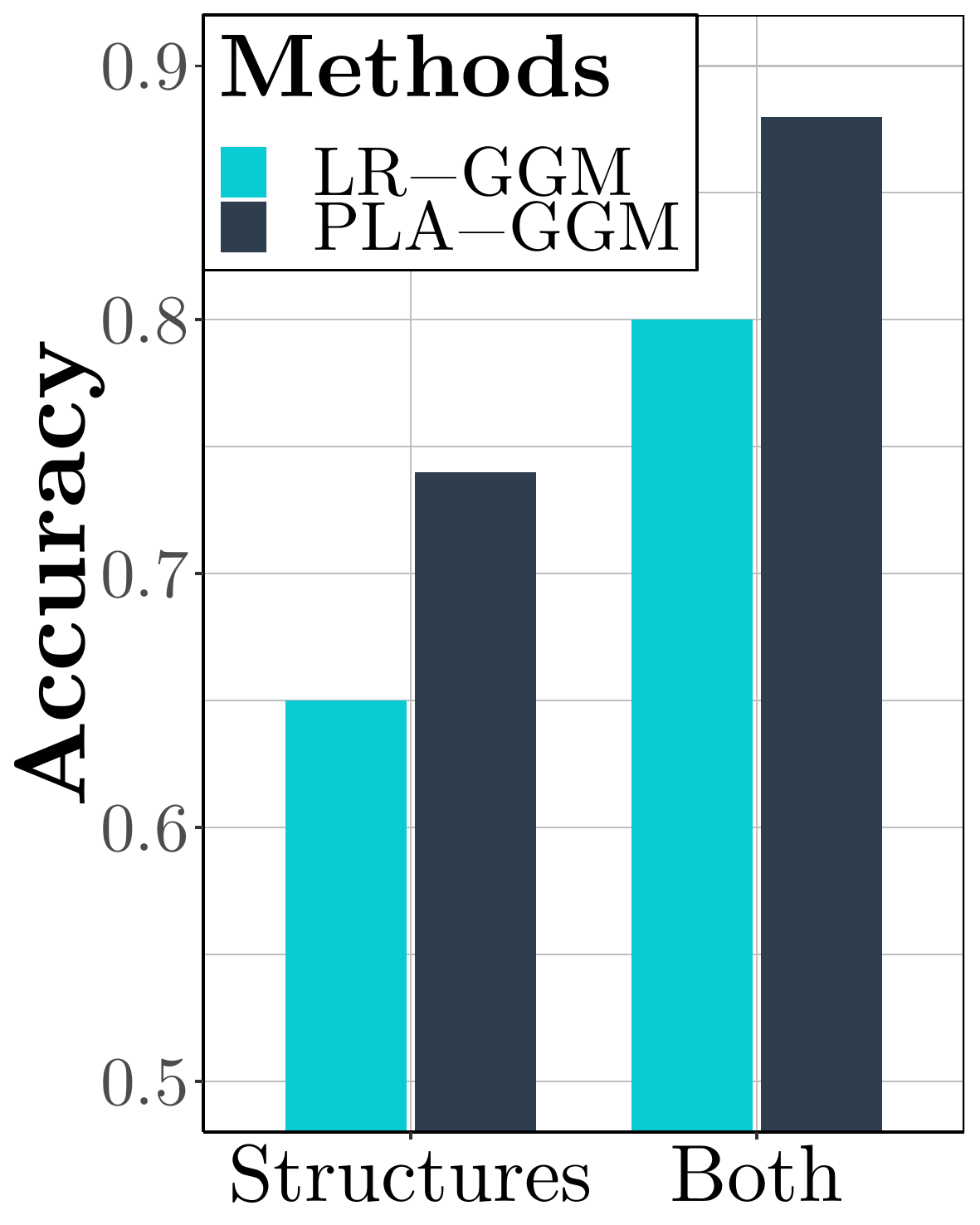}
		\caption{AUCs for the diagnosis of schizophrenia with the regularization parameters selected by AIC.}
		\label{fig:auc-aic}
	\end{minipage}
\vspace{-0cm}
\end{figure*}

\subsubsection*{General Analysis}
\label{sec:general-analysis}
We first generally analyze the brain functional connectivity by the PLA-GGM. Specifically, following the common practice in this area~\citep{belilovsky2016testing}, we assume that all the fMRIs from the subjects with schizophrenia follow a single PLA-GGM with the same brain functional connectivity, and thus combine the preprocessed fMRIs from the subjects into one dataset. Then, the PPL method is applied to the combined dataset to estimate an $\bOmega_0$, which corresponds to the brain functional connectivity for all the subjects with schizophrenia. The same procedure is also implemented on the control subjects' fMRI datsets. 

For a comparison, we also apply the LR-GGM discussed in Section~\ref{sec:competing-method} by the aforementioned procedure. The estimated brain functional connectivity for subjects and the controls with the two methods are reported in Figure~\ref{fig:glass-brain}. The ROIs are denoted by nodes with different colors. Edges among nodes denote the estimated functional connectivity. Red edges denote the positive connections, while the blues ones denote negative connections. The darker the color, the stronger the connection. We only provide the figure from one angle here. The figures from other angles are provided in the Supplements.

 Comparing the glass brain figure for controls with the one for subjects estimated by PLA-GGM, we find Occipital Pole and Central Opercular Cortex are the two areas differ the most. Interestingly, these two
areas have been implicated in the literature as highly
associated with schizophrenia \citep{sheffield2015fronto}. Also, by comparing the results of PLA-GGMs with those of LR-GGMs, the results of LR-GGMs are much denser and covered with lots of positive connections. This phenomenon is consistent with our analysis in Section~\ref{sec:competing-method}: LR-GGMs will treat the average of $\bOmega(g)$'s as the estimator to $\bOmega_0$ and derive incorrect over-dense estimates. This strongly suggests that the dense structure here is a result of the confounders which are not successfully accounted for by the regression.

\subsubsection*{Schizophrenia Diagnosis}
\label{sec:diagnosis}
Ideally, to demonstrate the accuracy of the proposed method for structure recovery, we should compare the estimated brain functional connectivity to the underlying ground truth, which, however, is not available in practice. Therefore, we consider a surrogate evaluation by using the estimated functional connectivity for schizophrenia diagnosis. Intuitively, if the recovered connectivity is more accurate due to effectively omitting the confounding, we should be able to improve schizophrenia diagnosis using the estimated connectivity as features. 
We apply PLA-GGMs and LR-GGMs to each subject respectively, and calculate $\hat{\bOmega}_0$'s for every subject. Then, we use $\hat{\bOmega}_0$'s as the input for classification methods to classify the subjects. For this two-class classification task, we use an $L_1$-regularized logistic regression as this is a common approach in the literature~\citep{patel2016classification}. 

We consider using only the structure (signs without values) of $\hat{\bOmega}_0$'s as the input for the classification. Although the classification is of course more challenging, we can more clearly see how helpful the structure itself is for the diagnosis. The AUCs for the diagnosis are reported in Figure~\ref{fig:auc-structure-class}. For a thorough comparison, we use the regularization parameters suggested by the R package glmnet for the penalized logistic regression and report all the AUCs. Clearly, PLA-GGMs resut in more accurate prediction. Therefore, the brain connectivity estimators derived by PLA-GGMs are more informative for schizophrenia diagnosis and more accurate than those of LR-GGMs. 
 
We next include the values into the input, and evaluate the accuracy for schizophrenia diagnosis. Again, we report all the results using different regularization parameters in Figure~\ref{fig:auc-all}. The green line indicates the best performance of using LR-GGM and $L_1$-regularized logistic regression achieved in \cite{patel2016classification}. Since we are using exactly the same dataset, we directly use their results for the LR-GGM combined with penalized logistic regression for a fair comparison. As a result, for most of the regularization parameters, PLA-GGMs derive more accurate diagnosis than LR-GGMs. Experiments using different $\mathds{1}_{\left\{ \abs{g}\geq g^* \right\}}$ are included in the Supplements. The results are similar. Therefore, we conclude that PLA-GGMs more accurately address confounding and derive more accurate estimators of brain functional connectivity.
 
We note, however, that some specialized alternative classifiers have been developed in the brain connectivity literature~\citep{ patel2016classification,arroyo2017network, andersen2018bayesian}, and expect that our approach will improve performance for those classifiers as well. We leave such further analysis to future work.

\section{Conclusions and Future Works}
We propose PLA-GGMS, to study the relationships among random variables with observed confounders. PLA-GGMs are especially generalized and facilitate $\sqrt{n}$-sparsisent estimators. The utility of PLA-GGMs is demonstrated
using a real-world fMRI dataset for the brain connectivity estimation. While we have been taking GGMs as an example, the results can be generalized to other undirected graphical models, especially the univariate exponential family distributions (UEFDs)~\citep{yang2015graphical}. We leave the details to future work.

\clearpage
\newpage
\bibliographystyle{icml2019}
\bibliography{MRI}

\newpage\clearpage\onecolumn
\newpage
\appendix
\addcontentsline{toc}{section}{Appendices}

\section*{Supplements}

\section{Proof of Theorem~\ref{lem:pl}}
Following the definition, we can derive the PL for the PLA-GGM as:
\begin{align*}
\begin{split}
\ell_{PL} \left(\left\{ \bz_i, g_i \right\}_{i \in [n]}; \bR(\cdot), \bOmega_0 \right)\propto & \sum_{i=1}^n \sum_{j = 1}^p \left\{ z_{ij} \left( \Omega_{ijj} + \sum_{j'\neq j} \Omega_{ijj'}z_{ij}z_{ij'} \right)
-\frac{1}{2}z_{ij}^2  \right.
\\&-\left.\frac{1}{2}\left( \Omega_{ijj} + \sum_{j'\neq j} \Omega_{ijj'}z_{ij}z_{ij'} \right)^2
\right\}. 
\end{split}
\end{align*} 
Then Lemma~\ref{lem:pl} can be proved by the definition of $\bz_{i,-j}$. 
\section{Proof of Lemma~\ref{lem:m}}
According to the analysis in Section~\ref{sec:pl}, we treat the PL as $p$ partially-linear additive linear regressions. Then, for each regression, we can derive $\hat{M}_{ij}$ as the estimation to the smooth part following the rationale in \cite{fan2005profile}. Combining the results for every regression, we can derive Lemma~\ref{lem:m}. 

\section{Proof of Theorem~\ref{thm:sparsistency}}
In this Section, we prove the $\sqrt{n}$-sparsistency of the $L_1$-regularized MPPLE by following the widely-used primal-dual witness proof technique \citep{wainwright2009sharp, ravikumar2010high, yang2011use, yang2015graphical}. PDW is characterized by the following Lemma~\ref{lem:dual-optimal}:

\begin{lemma}
	\label{lem:dual-optimal}
	Let $\hat{\bOmega}_0$ be an optimal solution to \eqref{eq:l1-mpple}, and $\hat{\bZ}$ be the corresponding dual solution. If $\hat{\bZ}$ satisfies  $\norm{\hat{\bZ}_{N}}_\infty < 1$, then any given optimal solution to \eqref{eq:l1-mpple} $\tilde{\bOmega}_0$  satisfies $\tilde{\bOmega}_{0I} = \bm{0}$. Moreover, if $\bH_{SS}$ is positive definite, then the solution to \eqref{eq:l1-mpple} is unique.
\end{lemma}
\begin{proof}
	Specifically, following the same rationale as Lemma 1 in \citealt{wainwright2009sharp},  Lemma 1 in \citealt{ravikumar2010high}, and Lemma 2 in \citealt{yang2011use}, we can derive Lemma \ref{lem:dual-optimal} characterizing the optimal solution of \eqref{eq:l1-mpple}.
\end{proof}

\subsection*{Bound $\norm{\grad F(\bOmega_0^*)}_\infty$}
Before we use the PDW, we first provide a Lemma bounding $\norm{\grad F(\bOmega_0^*)}_\infty$, which has been shown to be vital for PDW~\citep{wainwright2009sharp, ravikumar2010high, yang2011use, yang2015graphical}.

\begin{lemma}
	\label{lem:bound-der}
	Let $r :=  4C_5\lambda$. For any $\epsilon_d >0$, with probability of at least $1-\epsilon_d$, there exists $C_4>0$ and $N_d>0$ satisfying the following two inequalities:
	\begin{equation}
	\label{eq:bound-der-1}
	\norm{\grad F(\bOmega_0^*)}_\infty \leq C_4 \sqrt{\frac{\log p}{n}},
	\end{equation}
	\begin{equation}
	\label{eq:bound-der-2}
	\norm{\tilde{\bTheta}_S - \bTheta^*_S}_\infty \leq r,
	\end{equation}
	for $n>N_d$. 
\end{lemma}

\begin{proof} 
	We prove \eqref{eq:bound-der-1} and \eqref{eq:bound-der-2} in turn. 
	
	\subsubsection*{Proof of \eqref{eq:bound-der-1}}
	To begin with, we prove \eqref{eq:bound-der-1}. We define 
	
	\begin{equation*}
	\lambda_{ij}^* = \left( \bm{1}_{i} - \bS_{ij}\right)^\top \bx_j \bOmega^*_{0 \cdot j}.
	\end{equation*}
	We use the $F$ to denote the PPL defined in Definition~\ref{def:ppl}. 
	Then, the derivative of $F(\bOmega^*_{0 \cdot j})$ is:
	\begin{align}
	\label{eq:der-jk}
	\begin{split}
	\frac{\partial F(\bOmega^*_{0})}{\partial  \Omega^*_{0 j' j}} =& \frac{\sum_{i=1}^n\left\{
		-\left( \bm{1}_{i} - \bS_{ij}\right)^\top\by_j\left[\left( \bm{1}_{i} - \bS_{ij}\right)^\top \bx_j 
		\right]_{j'} +\lambda_{ij}^*\left[\left( \bm{1}_{i} - \bS_{ij}\right)^\top \bx_j 
		\right]_{j'}\right\}}{n}
	\\ &+\frac{\sum_{i=1}^n\left\{	-\left( \bm{1}_{i} - \bS_{ij'}\right)^\top\by_{j'}\left[\left( \bm{1}_{i} - \bS_{ij'}\right)^\top \bx_{j'} 
		\right]_{j} +\lambda_{ij'}^*\left[\left( \bm{1}_{i} - \bS_{ij}\right)^\top \bx_{j'} 
		\right]_{j}\right\}}{n},
	\end{split}
	\end{align}
	where $[\cdot]_j$ denotes the $j^\text{th}$ component of the vector.
	For the ease of presentation, we define
	\begin{equation*}
	\by_{j}' =
	\begin{bmatrix}
	\left( \bm{1}_{1} - \bS_{1j}\right)^\top\by_j
	\\\vdots
	\\ \left( \bm{1}_{n} - \bS_{nj}\right)^\top\by_j
	\end{bmatrix}
	\text{ and }
	\bx_{j}' =
	\begin{bmatrix}
	\left( \bm{1}_{1} - \bS_{1j}\right)^\top\bx_{j}
	\\\vdots
	\\ \left( \bm{1}_{n} - \bS_{nj}\right)^\top\bx_{j}
	\end{bmatrix}.
	\end{equation*}
	Then, we consider 
	\begin{equation}
	\label{eq:vector-regression}
	\frac{{\bx_j'}^\top \bx_j'\bOmega_{0\cdot j}^*-{\bx_j'}^\top\by_{j}'}{n}, 
	\end{equation}
	whose $j'^{\text{th}}$ component is just the target value
	\begin{equation*}
	\frac{\sum_{i=1}^n\left\{
		-\left( \bm{1}_{i} - \bS_{ij}\right)^\top\by_j\left[\left( \bm{1}_{i} - \bS_{ij}\right)^\top \bx_j 
		\right]_{j'} +\lambda_{ij}^*\left[\left( \bm{1}_{i} - \bS_{ij}\right)^\top \bx_j 
		\right]_{j'}\right\}}{n}.
	\end{equation*}
	
	Therefore, we focus on bounding \eqref{eq:vector-regression}. Then, 
	\begin{align*}
	\begin{split}
	\frac{{\bx_j'}^\top \bx_j'\bOmega_{0\cdot j}^*-{\bx_j'}^\top\by_{j}'}{n} =& \frac{{\bx_j'}^\top \bx_j' \left[ \bOmega_{0\cdot j}^* -  \left( {\bx_j'}^\top \bx_j' \right)^{-1}{\bx_j'}^\top\by_{j}' \right]}{n} 
	\\ = & \frac{{\bx_j'}^\top
		\left(\mathbf{I} - \bS_j  \right)(\bM_j + \bepsilon_j)
	}{n}
	\end{split},
	\end{align*}
	where the second equality is due to Lemma~\ref{lem:node-wise-normal}. 
	
	We fist study $\frac{{\bX_j'}^\top
		\left(\mathbf{I} - \bS_j  \right)\bM_j}{n}$: according ot Lemma~\ref{lem:bound-uniform}, for any $\epsilon_a > 0$, there exists $\delta_a>0$ and $N_a>0$ satisfying
	\begin{equation*}
	\text{P}\left\{ \norm{\frac{{\bx_j'}^\top
			\left(\mathbf{I} - \bS_j  \right)\bM_j}{n}   
	}_\infty
	> \delta_a\left[ \frac{\log(\frac{1}{h})}{nh} + h^4 
	+ 2h^2 \sqrt{\frac{\log(\frac{1}{h})}{nh} }
	\right] 
	\right\} < \epsilon_a,
	\end{equation*}
	with $n > N_a$. 
	
	According to Assumption~\ref{asm:order-n-h}
	\begin{equation}
	\label{eq:bound-re-1}
	\text{P}\left\{ \norm{\frac{{\bx_j'}^\top
			\left(\mathbf{I} - \bS_j  \right)\bM_j}{n}   
	}_\infty
	> \delta_aC_1 \sqrt{\frac{\log p}{n}}
	\right\} < \epsilon_a.
	\end{equation}

	Now, we study $\frac{{\bx_j'}^\top
		\left(\mathbf{I} - \bS_j  \right)\bepsilon_j}{n}$. According to Lemma~\ref{lem:bound-uniform-epsilon}, we have 
	\begin{align*}
	\frac{{\bx_j'}^\top
		\left(\mathbf{I} - \bS_j  \right)\bepsilon_j}{n}
	= \sum_{i=1}^n \left\{
	\bx_{ij} -\mathbb{E}^\top\left[ \mathds{1}_{g_{i'} > g^*}\bZ_{i,-j} \bZ_{i,-j}^\top \given g_i \right] \mathbb{E}^{-1}\left[ \mathds{1}^2_{g_{i'} > g^*}\bZ_{i,-j} \bZ_{i,-j}^\top \given g_i \right]\tilde{\bx}_{ij}
	\right\}
	\\ \epsilon_{ij}(1+ o_p(1))/n,
	\end{align*}
	uniformly for $j$.
	Note that $\frac{{\bx_j'}^\top
		\left(\mathbf{I} - \bS_j  \right)\bepsilon_j}{n}$is a $p \times 1$ vector. Therefore, for the ${j'}^\text{th}$ component, we have
	\begin{align}
	\label{eq:bound-component}
	\begin{split}
	&\abs{\left[ \frac{{\bx_j'}^\top
			\left(\mathbf{I} - \bS_j  \right)\bepsilon_j}{n}  \right]_{j'}}
	\\\leq& \abs{ \sum_{i:g^i \leq g^*} \left(1-\mathds{1}^2_{g^i>g^*}  \right) 
		z_{ij'}\epsilon_{ij} 	
	} 
	\left(1+ \abs{o_p(1)}  \right)/n
	\\ =&\frac{1}{2n} \abs{\sum_{i=1}\left(1- \mathds{1}^2_{g_i > g^*}   \right)\left[ \left( \frac{z_{ij'}+ \epsilon_{ij}}{\sqrt{2}} \right)^2 -1-
		\left( \frac{ z_{ij'}-\epsilon_{ij}}{\sqrt{2}} \right)^2 +1   \right] }(1+ \abs{o_p(1)})
	\end{split}.
	\end{align}
	It can be shown that $\left( \frac{z_{ij'}+ \epsilon_{ij}}{\sqrt{2}} \right)^2 $ and $
	\left( \frac{ z_{ij'}-\epsilon_{ij}}{\sqrt{2}} \right)^2$ are independent and follow chi-squared distribution with degree equal to $1$.
	
	By Lemma 1 in \cite{laurent2000adaptive}, the linear combinition of chi-squared random variables satisfies:
	\begin{equation*}
	\text{P} \left\{ \sum_{i=1}\left(1- \mathds{1}^2_{g_i > g^*}   \right)\left[ \left( \frac{z_{ij'}+ \epsilon_{ij}}{\sqrt{2}} \right)^2 -1   \right] \geq 2 \sqrt{nx} + 2\epsilon_c  \right\} \leq \exp(-\epsilon_c),
	\end{equation*}
	
	\begin{equation*}
	\text{P} \left\{ \sum_{i=1}\left(1- \mathds{1}^2_{g_i > g^*}   \right)\left[ \left( \frac{z_{ij'}+ \epsilon_{ij}}{\sqrt{2}} \right)^2 -1   \right] \leq -2 \sqrt{nx}  \right\} \leq \exp(-\epsilon_c),
	\end{equation*}
	
	\begin{equation*}
	\text{P} \left\{- \sum_{i=1}\left(1- \mathds{1}^2_{g_i > g^*}   \right)\left[ \left( \frac{z_{ij'}- \epsilon_{ij}}{\sqrt{2}} \right)^2 -1   \right] \leq -2 \sqrt{nx} -2\epsilon_c \right\} \leq \exp(-\epsilon_c),
	\end{equation*}
	and
	\begin{equation*}
	\text{P} \left\{- \sum_{i=1}\left(1- \mathds{1}^2_{g_i > g^*}   \right)\left[ \left( \frac{z_{ij'}- \epsilon_{ij}}{\sqrt{2}} \right)^2 -1   \right] \geq 2 \sqrt{nx} \right\} \leq \exp(-\epsilon_c),
	\end{equation*}
	for any $\epsilon_c>0$.
	Combing the previous four probabilistic bounds, we can derive
	\begin{align}
	\label{eq:bound-rhs}
	\begin{split}
	\text{P}\left\{ \sum_{i=1}\left(1- \mathds{1}^2_{g_i > g^*}   \right)\left[ \left( \frac{z_{ij'}+ \epsilon_{ij}}{\sqrt{2}} \right)^2 -1   \right]   \right.&    
	\\ - \sum_{i=1}\left(1- \mathds{1}^2_{g_i > g^*}   \right) &\left. \left[ \left( \frac{z_{ij'}- \epsilon_{ij}}{\sqrt{2}} \right)^2 -1   \right] \geq 4\sqrt{n\epsilon} + 2\epsilon_c
	\right\} \leq \exp(-2\epsilon_c)
	\end{split}
	\end{align} 
	and 
	\begin{align}
	\label{eq:bound-lhs}
	\begin{split}
	\text{P}\left\{ \sum_{i=1}\left(1- \mathds{1}^2_{g_i > g^*}   \right)\left[ \left( \frac{z_{ij'}+ \epsilon_{ij}}{\sqrt{2}} \right)^2 -1   \right]   \right.&    
	\\ - \sum_{i=1}\left(1- \mathds{1}^2_{g_i > g^*}   \right) &\left. \left[ \left( \frac{z_{ij'}- \epsilon_{ij}}{\sqrt{2}} \right)^2 -1   \right] \leq -4\sqrt{n\epsilon} - 2\epsilon_c
	\right\} \leq \exp(-2\epsilon_c)
	\end{split}.
	\end{align}
	Taking \eqref{eq:bound-rhs} and \eqref{eq:bound-lhs} into \eqref{eq:bound-component}, we can derive 
	\begin{align}
	\label{eq:bound-prob-1}
	\begin{split}
	\text{P} \left\{\abs{\left[ \frac{{\bX_j'}^\top
			\left(\mathbf{I} - \bS_j  \right)\bepsilon_j}{n}  \right]_{j'}} \geq \left(2\sqrt{\frac{\epsilon_c}{n}} + \frac{\epsilon_c}{n}\right)(1+ \abs{o_p(1)})
	\right\} \leq 2\exp(-2\epsilon_c)
	\end{split}.
	\end{align}
	
	Then, by the definition of $o_p(1)$, for any $\epsilon_b>0$, there exists $N_b$ so that for $n>N_b$:
	\begin{equation}
	\label{eq:bound-prob-2}
	\text{\normalfont P}\left\{ \abs{o_p(1)} \geq 1 \right\} \leq \epsilon_b.
	\end{equation}

	Combining \eqref{eq:bound-prob-1} and \eqref{eq:bound-prob-2}, we derive
	\begin{align}
	\label{eq:bound-re-2}
	\begin{split}
	\text{P} \left\{\norm{\frac{{\bX_j'}^\top
			\left(\mathbf{I} - \bS_j  \right)\bM_j}{n}   
	}_\infty \geq \left(4\sqrt{\frac{\epsilon_c}{n}} + 2\frac{\epsilon_c}{n}\right)
	\right\} \leq 2p\exp(-2\epsilon_c) + \epsilon_b
	\end{split},
	\end{align}
	by a union bound. 
	Eventually, according to \eqref{eq:bound-re-1} and \eqref{eq:bound-re-2}, and by setting $\epsilon_c = 2 \log p$ we prove:
	\begin{equation*}
	\norm{
		\frac{{\bx_j'}^\top \bx_j'\bOmega_{0\cdot j}^*-{\bx_j'}^\top\by_{j}'}{n}}_\infty \leq (6+ \delta_a C_1)\sqrt{\frac{2\log p}{n}},
	\end{equation*}
	with probability larger than $1 - \epsilon_b - \epsilon_a - 2p^{-1}$.
	Thus, for any $\epsilon_d>0$, there exists $C_4>0$ and $N_d >0$
	\begin{equation*}
	\norm{\grad F(\bTheta^*)}_\infty \leq C_4\sqrt{\frac{\log p}{n}},
	\end{equation*}
	with probability larger than $1 - \epsilon_d$, for $n>N_d$.
	\subsubsection*{Proof of \eqref{eq:bound-der-2}}
	
	To prove \eqref{eq:bound-der-2}, we use the fixed point method by defining a map $G(\bm{\Delta}_S) := - \bH_{SS}^{-1} \left[ \grad_S F(\bOmega^*_{0S} + \bm{\Delta}_S )+ \lambda \hat{\bZ}_{S} \right] + \bm{\Delta}_{S}$.
	If $\norm{\bm{\Delta}}_\infty \leq r$, by Taylor expansion of $\grad_S F(\bOmega_0^* + \bm{\Delta})$ centered at $\grad_SF(\bOmega_0^*)$,
	\begin{align*}
	\Norm{G(\bm{\Delta}_{S})}_\infty \hspace{-2mm} = & \Norm{- \bH_{SS}^{-1} \left[ \grad_S F(\bOmega^*_{0S} ) + \bH_{SS} \bDelta_S + \bm{R}_S(\bDelta) + \lambda \hat{\bZ}_{S} \right] \hspace{-1mm} + \hspace{-1mm} \bm{\Delta}_{S} }_\infty \hspace{-2mm} 
	\\=& \Norm{ - \bH_{SS}^{-1} \left( \grad_S F(\bOmega^*_{0S}) + \bm{R}_S(\bm{\Delta}) +\lambda \hat{\bZ}_{S} \right)}_\infty\\
	\leq & \Norm{\bH_{SS}^{-1}}_\infty \left(\norm{\grad_S F(\bOmega^*_{0S}) }_\infty + \norm{R_S(\bm{\Delta})}_\infty + \lambda \norm{\hat{\bZ}_{S}}_\infty \right)
	\\	\leq & C_2(\lambda+C_3r^2+\lambda) = C_2C_3r^2 + 2C_2\lambda,
	\end{align*}
	where the inequality is due to Assumption~\ref{asm:incoherence} and Assumption~\ref{asm:fix}, and $\norm{\grad_S F(\bTheta^* )}_\infty \leq \lambda$ with a high probability, according to \eqref{eq:bound-der-1}. Then, based on the definition of $r$, we can derive the upper bound of $\Norm{G(\bm{\Delta}_{S})}_\infty$ as $\Norm{G(\bm{\Delta}_{S})}_\infty \leq r/2 + r/2 = r$.
	
	Therefore, according to the fixed point theorem \citep{ortega2000iterative, yang2011use}, there exists $\bm{\Delta}_S$ satisfying  $G(\bm{\Delta}_S) = \bm{\Delta}_S$, which indicates $\grad_S F(\bOmega_0^* + \bm{\Delta} ) + \lambda \hat{\bZ}_{S} = \bm{0}$.
	The optimal solution to \eqref{eq:restricted} is unique, and thus 
	$\tilde{\bm{\Delta}}_{S} = \bm{\Delta}_S$. Therefore, $\norm{\tilde{\bm{\Delta}}_S}_\infty \leq r$ , with probability larger than $1-\epsilon$.
\end{proof}

\subsection*{PDW}
\label{sec:main-proof}
By Lemma \ref{lem:dual-optimal}, we can prove the sparsistency by building an optimal solution to \eqref{eq:l1-mpple} satisfying the strict dual feasibility (SDF) defined as $\norm{\hat{\bZ}_N}_\infty <1$, which is summarized. Therefore, we now build a solution by solving a restricted problem.

\subsubsection*{Solve a Restricted Problem}
First of all, we derive the KKT condition of \eqref{eq:l1-mpple}:
\begin{equation}
\label{eq:kkt}
\grad F(\hat{\bOmega}_0) + \lambda \hat{\bZ} = \bm{0}.
\end{equation}

To construct an optimal primal-dual pair solution, we define $\tilde{\bOmega}_0$ as an optimal solution to the restricted problem:
\begin{align}
\label{eq:restricted}
\begin{split}
\tilde{\bOmega}_0:= \min_{\bOmega_0}  F(\bOmega_0)+\lambda \norm{\bOmega_0}_1,
\end{split}
\end{align}
with ${\bOmega_0}_N = \mathbf{0}$. $\tilde{\bOmega}_0$ is unique due to Lemma \ref{lem:dual-optimal}.  Then, we define the subgradient  corresponding to $\tilde{\bOmega}_0$  as $\tilde{\bZ}$. Therefore, $(\tilde{\bOmega}_0,\tilde{\bZ})$ is a pair of optimal solutions to the restricted problem \eqref{eq:restricted}. $\tilde{\bZ}_S$ is determined according to the values of $\tilde{\bOmega}_{0S}$ via the KKT conditions of \eqref{eq:restricted}. Thus we have
\begin{equation}
\label{eq:kkt-s}
\grad_S F(\tilde{\bTheta}) + \lambda \tilde{\bZ}_S = \bm{0},
\end{equation}
where $\grad_S$ represents the gradient components with respect to $S$. Letting $\hat{\bOmega}_0 = \tilde{\bOmega}_0$, we determine $\tilde{\bZ}_N$ according to \eqref{eq:kkt}. It now remains to show that $\tilde{\bZ}_N$ satisfies SDF.

\subsubsection*{SDF}
Now, we demonstrate that $\tilde{\bTheta}$ and $\tilde{\bZ}$ satisfy SDF. We define $\tilde{\bm{\Delta}} := \tilde{\bTheta} - \bTheta^*$. By \eqref{eq:kkt-s}, and by the Taylor expansion of $\grad_S F(\tilde{\bOmega}_0)$, we have that 
\begin{equation*}
\bH_{SS} \tilde{\bm{\Delta}}_{S} + \grad_S F(\bOmega^*_0) + \bm{R}_S(\tilde{\bm{\Delta}}) +\lambda \tilde{\bZ}_{S} = \bm{0},
\end{equation*} 
which means 
\begin{equation}
\label{eq:tilde-delta-s} \tilde{\bm{\Delta}}_{S} =\bH_{SS}^{-1} \left[-\grad_S F(\bOmega_0^*) - \bm{R}_S(\bm{\tilde{\Delta}}) - \lambda \tilde{\bZ}_{S}\right],
\end{equation}
where $\bH_{SS}$ is positive definite and hence invertible. 

By the definition of $\tilde{\bOmega}_0$ and $\tilde{\bZ}$,
\begin{equation}
\label{eq:delta}
\grad F(\tilde{\bOmega}_0) + \lambda \tilde{\bZ} = \mathbf{0} \Rightarrow \grad F(\bOmega_0^*) + \bH \tilde{\bDelta} + \bm{R}(\tilde{\bOmega}_0) + \lambda\tilde{\bZ} = \bm{0} \Rightarrow \grad_N F(\tilde{\bTheta})  + \bH_{NS} \tilde{\bDelta}_S+ \bm{R}_N(\tilde{\bDelta}) + \lambda \tilde{\bZ}_N=\bm{0}.
\end{equation}

Due to \eqref{eq:tilde-delta-s}, 
\begin{align*}
\begin{split}
\lambda \norm{\tilde{\bZ}_{N}}_\infty 
= & \norm{-\bH_{NS} \tilde{\bm{\Delta}}_{S}- \grad_N F(\bOmega_0^*) -\bm{R}_N(\tilde{\bm{\Delta}})}_\infty \nonumber \\
\leq &\Norm{\bH_{NS}   \bH_{SS}^{-1} \left[-\grad_S F(\bOmega_0^*) - \bm{R}_S(\tilde{\bm{\Delta}} )- \lambda \tilde{\bZ}_{S}\right]}_\infty + \norm{\grad_N F(\bOmega_0^*) + \bm{R}_N(\tilde{\bm{\Delta}})}_\infty
\nonumber \\
\le & \Norm{\bH_{NS}   \bH_{SS}^{-1} }_{\infty} \Norm{\grad_S F(\bOmega_0^*) + \bm{R}_S(\tilde{\bm{\Delta}})}_{\infty} + \Norm{\bH_{NS} \bH_{SS}^{-1} }_{\infty} \Norm{\lambda \tilde{\bZ}_{S}}_{\infty} + \norm{\grad_N F(\bOmega_0^*) + \bm{R}_N(\tilde{\bm{\Delta}})}_\infty \nonumber 
\end{split}.
\end{align*}
Further, we use the Assumption~\ref{asm:incoherence}, 
\begin{align}
\label{eq:tilde-z-i-1}
\lambda \norm{\tilde{\bZ}_{N}}_\infty 
\le & (1-\alpha) \left(  \norm{\grad_S F(\bOmega_0^*)}_\infty+\norm{\bm{R}_S(\tilde{\bm{\Delta}})}_\infty  \right) + (1-\alpha)\lambda + \left(  \norm{\grad_N F(\bOmega_0^*)}_\infty + \norm{\bm{R}_N (\tilde{\bm{\Delta}})}_\infty  \right) \nonumber \\
\leq & (2-\alpha) \left(  \norm{\grad F(\bOmega_0^*)}_\infty+\norm{\bm{R}(\tilde{\bm{\Delta}})}_\infty  \right) + (1-\alpha)\lambda,
\end{align}
where we have used  in the first inequality, and the third inequality is due to Assumption \ref{asm:incoherence}. 

Now, we study $\norm{\grad F(\bOmega_0^*)}_\infty$.By Lemma \ref{lem:bound-der} and the assumption on $\lambda$ in Theorem~\ref{thm:sparsistency},  $\norm{\grad F(\bTheta^*)}_\infty \leq \frac{\alpha C_4}{4} \sqrt{\frac{\log p}{n}}\leq \frac{\alpha \lambda}{4} $, with probability larger than $1-\epsilon_d$.  

It remains to control $\norm{\bm{R}(\tilde{\bm{\Delta}})}_\infty$. According to Assumption~\ref{asm:fix} and Lemma~\ref{lem:bound-der}, 
\begin{equation}
\label{eq:r-tilde-delta-inf}
\norm{\bm{R}(\tilde{\bm{\Delta}})}_\infty
\leq C_3 \norm{\bDelta}_\infty^2 \leq  C_3 r^2 \leq C_3 (4C_2\lambda)^2 = \lambda \frac{64C_2^2 C_3}{\alpha} \frac{\alpha \lambda}{4} \le \left(C_5 \sqrt{\frac{\log p}{n}}\right) \frac{64C_2^2 C_3}{\alpha}  \frac{\alpha \lambda}{4},
\end{equation}
where in the last inequality we have used the assumption $\lambda \leq C_5\sqrt{\frac{\log p}{n}}$ in Theorem~\ref{thm:sparsistency}. Therefore, when we choose $n \ge \left(64C_5  C_2^2C_3/\alpha\right)^2\log p$ in Theorem~1, from \eqref{eq:r-tilde-delta-inf}, we can conclude that $\norm{\bm{R}(\tilde{\bm{\Delta}})}_\infty \leq \frac{\alpha \lambda }{4}$. As a result, $\lambda \norm{\hat{\bZ}_{N}}_\infty $ can be bounded by $\lambda \norm{\tilde{\bZ}_{N}}_\infty   < \alpha\lambda/2+\alpha\lambda/2 +(1-\alpha)\lambda = \lambda$.
Combined with Lemma \ref{lem:dual-optimal}, we demonstrate that any optimal solution of \eqref{eq:l1-mpple} satisfies $\tilde{\bTheta}_{N} = \bm{0}$. Furthermore, \eqref{eq:bound-der-2} controls the difference between the optimal solution of \eqref{eq:l1-mpple} and the real parameter by $\norm{\tilde{\bm{\Delta}}_S}_\infty \leq r$, by the fact that $r \le \norm{\bTheta^*_S}_\infty$ in Theorem~\ref{thm:sparsistency}, $\hat{\bTheta}_S$ shares the same sign with $\bTheta^*_S$. 

\subsection*{Auxiliary Lemmas}
In this section, we provide and prove the used auxiliary lemmas. 
\begin{lemma}
	\label{lem:node-wise-normal}
	For the graphical model defined in Section~\ref{sec:model} parameterized by $\bOmega_0^*$, the conditional distribution of $Z_{ij}$ follows
	\begin{equation*}
	\left(Z_{ij}
	\given G_i = g_i \right)\sim \bZ_{i,-j}^\top \bOmega_{0\cdot j} +M_{ij}+ \epsilon_{ij},
	\end{equation*}
	where 
	\begin{equation*}
	\left[\bZ_{i,-j} \right]_{j'}=
	\begin{cases}
	Z_{ij'} & j'\neq j
	\\1 & j'=j
	\end{cases}.
	\end{equation*}
	$\epsilon_{ij}$'s follow the standard normal distribution, and $\epsilon_{ij}$ is independent with $\epsilon_{i'j}$ for $j\neq j' \in [p]$. 
\end{lemma}
\begin{proof}
	According to Lemma~\ref{lem:pl}, the node-wise conditional distribution of a PLA-GGM follows a Gaussian distribution. Then, Lemma~\ref{lem:node-wise-normal} can be proved.
\end{proof}

\begin{lemma}
	\label{lem:kernel}
	For a kernel regression on $\left\{x_i,y_i \right\}_{i=1}^n$ as the IID samples of $(X,Y)$. Assume that $\EE \abs{Y}^s < \infty$ and $\sup_X \in \abs{Y}^s f(X,Y)dY \leq \infty$. Given that $n^{2\epsilon -1}h \to \infty$ for $\epsilon < 1-s^{-1}$, we have
	
	\begin{equation*}
	\sum_x\abs{\frac{1}{n} \sum_{i=1}^n \left[K_h(x_i - x) -\EE\left\{ K_h(x_i-x)y_i   \right\}  \right]   
	}
	=O_p\left( \left\{ \frac{\log(1/h)}{nh}   
	\right\}^{1/2} \right).
	\end{equation*}
\end{lemma}
\begin{proof}
	Lemma~\ref{lem:kernel} follows \cite{mack1982weak}.
\end{proof}

\begin{lemma}
	\label{lem:sub-gaussian}
	Suppose $\bY = \left\{ Y_1,Y_2 \cdots, Y_n  \right\}$ follows a multivariate Gaussian distribution, then $\max\abs{Y_i}$ follows a sub-Gaussian distribution with variance $\max \text{var}(Y_i)$. Further, for any $t>0$, the tail probability can be controlled via
	\begin{equation*}
	\text{\normalfont P}\left\{ \max\abs{\epsilon_{ij}} \geq t \right\} \leq \exp\left( \frac{-t^2}{2} \right).
	\end{equation*}
\end{lemma}

\begin{lemma}
	\label{lem:bound-uniform}
	For any $\epsilon >0$, there exists $\delta >0$ and $N>0$, so that when $n>N$, we have
	\begin{equation*}
	\text{\normalfont P}\left\{\norm{\frac{\bX_{j}'(\bI -\bS_j )\bM_j}{n}}_\infty \geq \delta c^2_n\right\} \leq \epsilon,
	\end{equation*}
	uniformly for $j\in[p]$.
\end{lemma}
\begin{proof}
	To start with, we review the definition of $\bS_{ij}$
	\begin{equation*}
	\bS_{ij} = 
	\begin{bmatrix}
	\mathds{1}_{g_{i'} > g^*}\bz_{i,-j}^\top & 0 
	\end{bmatrix}
	\left(
	\bD_{ij}^\top \bW_i \bD_{ij}
	\right)^{-1}
	\bD_{ij}^\top \bW_i. 
	\end{equation*}
	We first study $\bD_{ij}^\top \bW_i \bD_{ij}$:
	\begin{equation*}
	\bD_{ij}^\top \bW_i \bD_{ij} = 
	\begin{bmatrix}
	\sum_{i'=1}^n \mathds{1}^2_{g_{i'} > g^*} \bz_{i',-j} \bz_{i',-j}^\top \psi\left(\abs{g_{i'}-g_i}/h\right) 
	& \sum_{i'=1}^n \mathds{1}^2_{g_{i'} > g^*}\bz_{i',-j} \bz_{i',-j}^\top \frac{g_{i'} - g_i}{h} \psi\left(\abs{g_{i'}-g_i}/h\right)
	\\\sum_{i'=1}^n \mathds{1}^2_{g_{i'} > g^*}\bz_{i', -j} \bz_{i'-j}^\top \frac{g_{i'} - g_i}{h} \psi\left(\abs{g_{i'}-g_i}/h\right)
	&\sum_{i'=1}^n \mathds{1}^2_{g_{i'} > g^*}\bz_{i', -j} \bz_{i',-j}^\top \left(\frac{g_{i'} - g_i}{h}\right)^2 \psi\left(\abs{g_{i'}-g_i}/h\right)
	\end{bmatrix}.
	\end{equation*}
	To bound $\bD_{ij}^\top \bW_i \bD_{ij}$ uniformly over $j$, we consider a random vector $\bB_i = [\mathds{1}_{g_{i'} > g^*} \bZ_i^\top, 1]^\top$, with observations
	\begin{equation*}
	\begin{bmatrix}
	\bb_1 = \left[ \mathds{1}_{g_{i'} > g^*}\bz_1^\top,1  \right]
	\\ \vdots
	\\ \bb_n = \left[ \mathds{1}_{g_{i'} > g^*}\bz_n^\top,1 \right]
	\end{bmatrix}.
	\end{equation*}
	Then, we study an auxiliary matrix
	\begin{equation*}
	\bO_i = 
	\begin{bmatrix}
	\sum_{i'=1}^n \mathds{1}^2_{g_{i'} > g^*} \bb_{i'} \bb_{i'}^\top \psi\left(\abs{g_{i'}-g_i}/h\right) 
	& \sum_{i'=1}^n \mathds{1}^2_{g_{i'} > g^*}\bb_{i'} \bb_{i'}^\top \frac{g_{i'} - g_i}{h} \psi\left(\abs{g_{i'}-g_i}/h\right)
	\\\sum_{i'=1}^n \mathds{1}^2_{g_{i'} > g^*}\bb_{i'} \bb_{i'}^\top \frac{g_{i'} - g_i}{h} \psi\left(\abs{g_{i'}-g_i}/h\right)
	&\sum_{i'=1}^n \mathds{1}^2_{g_{i'} > g^*}\bb_{i'} \bb_{i'}^\top \left(\frac{g_{i'} - g_i}{h}\right)^2 \psi\left(\abs{g_{i'}-g_i}/h\right)
	\end{bmatrix}.
	\end{equation*}

	Therefore, the components of $\bD_{ij}^\top \bW_i \bD_{ij}$ belong to $\bO_i$, and each part of $\bO_i$ is in the form of a kernel regression. By Lemma~\ref{lem:kernel}, we have
	\begin{equation*}
	\bO_i = nf(g_i)\mathbb{E}\left[ \bB_{i}\bB_{i}^\top \given g_i \right] \otimes	
	\begin{bmatrix}
	1&0
	\\0&\mu_2
	\end{bmatrix}\left\{1+O_p(c_n) \right\},
	\end{equation*} 
	which holds uniformly for $i$. Therefore,
	\begin{equation}
	\label{eq:dwd}
	\bD_{ij}^\top \bW_i \bD_{ij} = nf(g_i)\mathbb{E}\left[ \mathds{1}^2_{g_{i'} > g^*}\bZ_{i,-j} \bZ_{i,-j}^\top \given g_i \right] \otimes	
	\begin{bmatrix}
	1&0
	\\0&\mu_2
	\end{bmatrix}\left\{1+O_p(c_n) \right\}
	\end{equation} 
	holds uniformly for $i$ with the same $O_p(c_n)$ for every $j$.
	Define 
	\begin{equation*}
	\balpha_j(g_i) = 
	\begin{bmatrix}
	\bOmega_{1\cdot j} &\cdots&\bOmega_{n\cdot j}
	\end{bmatrix}.
	\end{equation*}
	By the same technique, uniformly for $i$ and with the same $O_p(c_n)$ for every $j$, we can show
	\begin{equation}
	\label{eq:dwm}
	\bD_{ij}^\top \bW_i \bM_{j} = nf(g_i)\mathbb{E}\left[ \mathds{1}^2_{g_{i'} > g^*}\bZ_{i,-j} \bZ_{i,-j}^\top \given g_i \right] \otimes	
	\begin{bmatrix}
	1&0
	\end{bmatrix}^\top\balpha_j(g_i)
	\left\{1+O_p(c_n) \right\},
	\end{equation}
	and
	\begin{equation}
	\label{eq:dwx}
	\bD_{ij}^\top \bW_i \bx_{j} = nf(g_i)\mathbb{E}\left[ \mathds{1}_{g_{i'} > g^*}\bZ_{i,-j} \bZ_{i,-j}^\top \given g_i \right] \otimes	
	\begin{bmatrix}
	1&0
	\end{bmatrix}^\top\left\{1+O_p(c_n) \right\}.
	\end{equation}  
	
	Combining \eqref{eq:dwd} and \eqref{eq:dwm} we have
	\begin{equation}
	\label{eq:aux-1}
	\begin{bmatrix}
	\tilde{\bx}_j^\top & 0 
	\end{bmatrix}
	\left(
	\bD_{ij}^\top \bW_i \bD_{ij}
	\right)^{-1}
	\bD_{ij}^\top \bW_i \bM_j = \tilde{\bx}_j^\top\balpha_j(g_i) \left\{1+O_p(c_n) \right\}.
	\end{equation}

	Similarly, combining \eqref{eq:dwd} and \eqref{eq:dwx}, we have 
	\begin{equation}
	\label{eq:aux-2}
	\bx'_{ij} = \bx_{ij}-   \tilde{\bx}_{ij} \mathbb{E}^{-1}\left[ \mathds{1}^2_{g_{i'} > g^*}\bZ_{i,-j} \bZ_i^\top \given g_i \right]\mathbb{E}\left[ \mathds{1}_{g_{i'} > g^*}\bZ_{i,-j} \bZ_{i,-j}^\top \given g_i \right].
	\end{equation}
	Next, we follow the rationale of the Lemma A.4 in \citep{fan2005profile}, and combine \eqref{eq:aux-1} and \eqref{eq:aux-2}. Finally, we have
	\begin{equation*}
	\frac{\bx_{j}'(\bI -\bS_j )\bM_j}{n} = O_p(c^2_n)
	\end{equation*}
	uniformly for $j$.
\end{proof}

\begin{lemma}
	\label{lem:bound-uniform-epsilon}
	For any $\epsilon >0$, there exists $N>0$, so that when $n>N$, we have
	\begin{equation*}
	\norm{\bx_j'^\top (\bI - \bS_j)\bepsilon_j}_\infty \geq
	2\sum_{i=1}^n \left\{
	\bx_{ij} -\mathbb{E}^\top\left[ \mathds{1}_{g_{i'} > g^*}\bZ_{i,-j} \bZ_{i,-j}^\top \given g_i \right] \mathbb{E}^{-1}\left[ \mathds{1}^2_{g_{i'} > g^*}\bZ_{i,-j} \bZ_i^\top \given g_i \right]\tilde{\bx}_{ij}
	\right\}\epsilon_{ij},
	\end{equation*}
	uniformly for $j\in[p]$ with probability less than $\epsilon$.
\end{lemma}

\begin{proof}
	
	By definition, we have
	\begin{equation*}
	\bx_j'^\top (\bI - \bS_j)\bepsilon_j = 
	\sum_{i=1}^n \bx_{ij}'
	\left\{
	\epsilon_{ij} - 
	\begin{bmatrix}
	\tilde{\bx}_{ij}^\top & 0 
	\end{bmatrix}
	\left(
	\bD_{ij}^\top \bW_i \bD_{ij}
	\right)^{-1}\bD_{ij}^\top \bW_i \bepsilon_j
	\right\}.
	\end{equation*}
	Using the technique in \eqref{eq:dwd}, we have
	\begin{equation*}
	\begin{bmatrix}
	\tilde{\bx}_{ij}^\top & 0 
	\end{bmatrix}
	\left(
	\bD_{ij}^\top \bW_i \bD_{ij}
	\right)^{-1}\bD_{ij}^\top \bW_i \bepsilon_j = \tilde{\bx}_{ij}^\top\mathbb{E}^{-1}\left[ \mathds{1}^2_{g_{i'} > g^*}\bZ_{i,-j} \bZ_i^\top \given g_i \right] 
	\mathbb{E} \left[ \tilde{\bx}_{ij}^\top \given g_i \right]O_p(c_n).
	\end{equation*}
	Therefore, 
	\begin{equation*}
	\bx_j'^\top (\bI - \bS_j)\bepsilon_j = 
	\sum_{i=1}^n \left\{
	\bx_{ij} -\mathbb{E}^\top\left[ \mathds{1}_{g_{i'} > g^*}\bZ_{i,-j} \bZ_{i,-j}^\top \given g_i \right] \mathbb{E}^{-1}\left[ \mathds{1}^2_{g_{i'} > g^*}\bZ_{i,-j} \bZ_i^\top \given g_i \right]\tilde{\bx}_{ij}
	\right\}\epsilon_{ij} [1+o_p(1)],
	\end{equation*}
	uniformly for $j$.
\end{proof}

\section{Proof of Theorem~\ref{thm:con-lr-ggm}}
We first study CON-GGMs.
According to \eqref{eq:con-ggm} and \cite{eaton1983multivariate}, we have 
\begin{equation*}
\left[\text{cov}\left(\bZ\given G=g   \right)\right]^{-1} = \left[
\bSigma_{\bZ \bZ}  
- \bSigma_{\bZ G}
\bSigma_{GG}^{-1}
\bSigma_{G \bZ} 
\right]^{-1},
\end{equation*}
whose right-hand side has nothing to do with $g$. Therefore, the conditional distribution of $\bZ\given G=g $ follows a GGM with parameter $\left[
\bSigma_{\bZ \bZ}  
- \bSigma_{\bZ G}
\bSigma_{GG}^{-1}
\bSigma_{G \bZ} 
\right]^{-1}$ irrelevant to $g$. In other words CON-GGM is equivalent to assuming that $G$ follows a normal distribution and $\bR(g)=0$ on the basis of the proposed PLA-GGM. 

Then, we study LR-GGMs. Again, given $G=g$ for any $g$, we have
\begin{equation*}
\left[\text{cov}\left(\bZ\given G=g   \right)\right]^{-1} = \bOmega_0,
\end{equation*}
which has nothing to do with $G$ either. Given $G=g$, the conditional distribution of $\bZ\given G=g$ follows a GGM with the parameter $\bOmega$. Therefore, LR-GGM is a special case of the proposed PLA-GGM by assuming $\bR(g)=0$.

\section{Experiments}

\subsection*{Data Simulation}
To simulate the samples from PLA-GGMs, we first define 
\begin{equation*}
f(g) = \begin{cases}
g-10 & g>12
\\x+\frac{(x-12)^2}{4}-11 & 10< g \leq 12
\\ 0 & -10 < g \leq 10
\\ x +\frac{(x+12)^2}{4} + 11 & -12 < g \leq -10
\\ g+10 & g\leq -12
\end{cases} 
\end{equation*}
We provide the following procedure:
\begin{enumerate}
	\item We consider $p = 10, 20, 50, 100$, and implement the following steps separately.  
	\item We randomly generate a sparse precision matrix as $\Omega_0$ Specifically, each element of $\Omega_0$ is drawn randomly to be non-zero with probability 0.3.
	\item A dense precision matrix $\bW$ is generated to build the confounding. 
	\item We take $\left\{-400,\cdots, 0,\cdots, 399  \right\}$ as the confounders. For each $g \in \left\{-400,\cdots, 0,\cdots, 399  \right\}$, the precision matrix is selected to be $\bOmega(g) = \bOmega_0 + f(g)\bW$, and a sample is generated by a GGM with parameter $\bOmega(g)$. Thus, we get 800 samples.  
\end{enumerate} 

Note that the procedure is equivalent to selecting $g^* =10$. 

\subsection*{Glass Brains for Brain Function Connectivity Estimation}
We report the glass brains from other angles for the brain function connectivity estimation experiment in Section~\ref{sec:general-analysis}.

\begin{figure*}[h]
	\centering
	\includegraphics[scale=0.8]{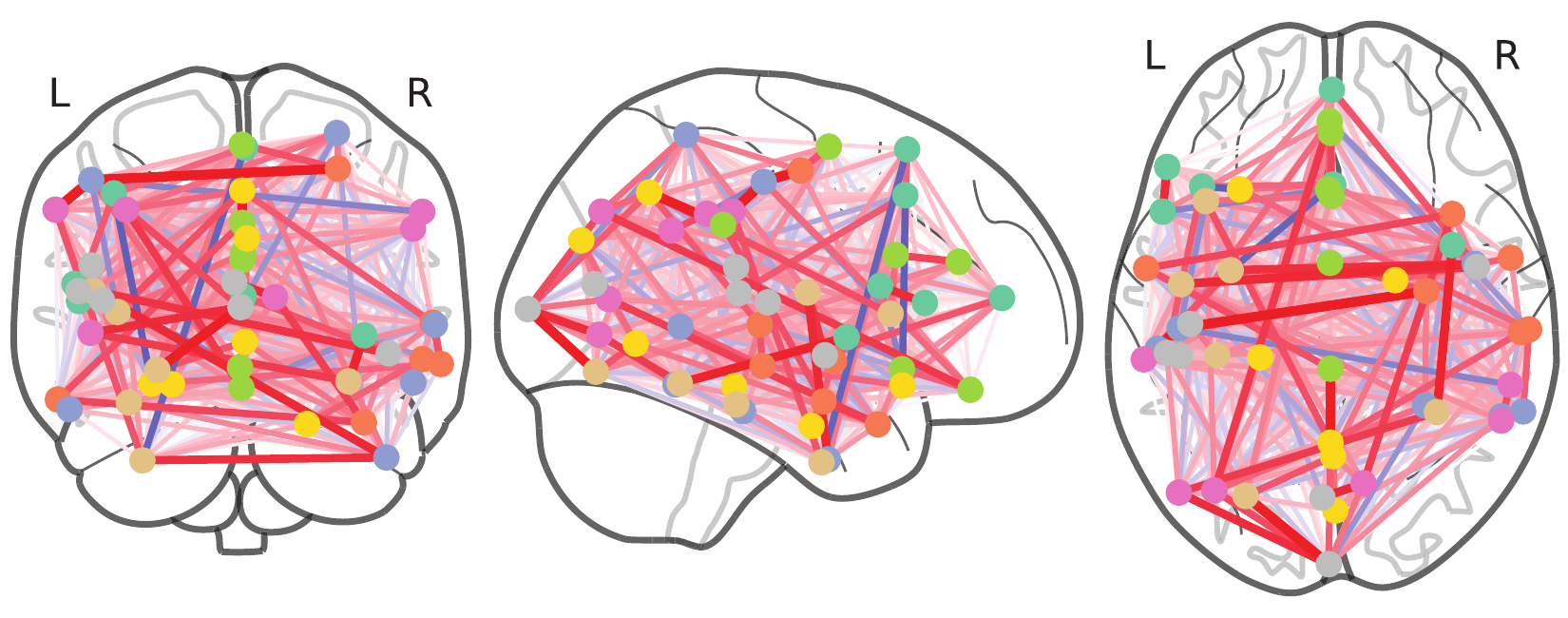}
	\caption{Controls using PLA-GGMs}
\end{figure*}
\begin{figure*}[h]
	\centering
	\includegraphics[scale=0.8]{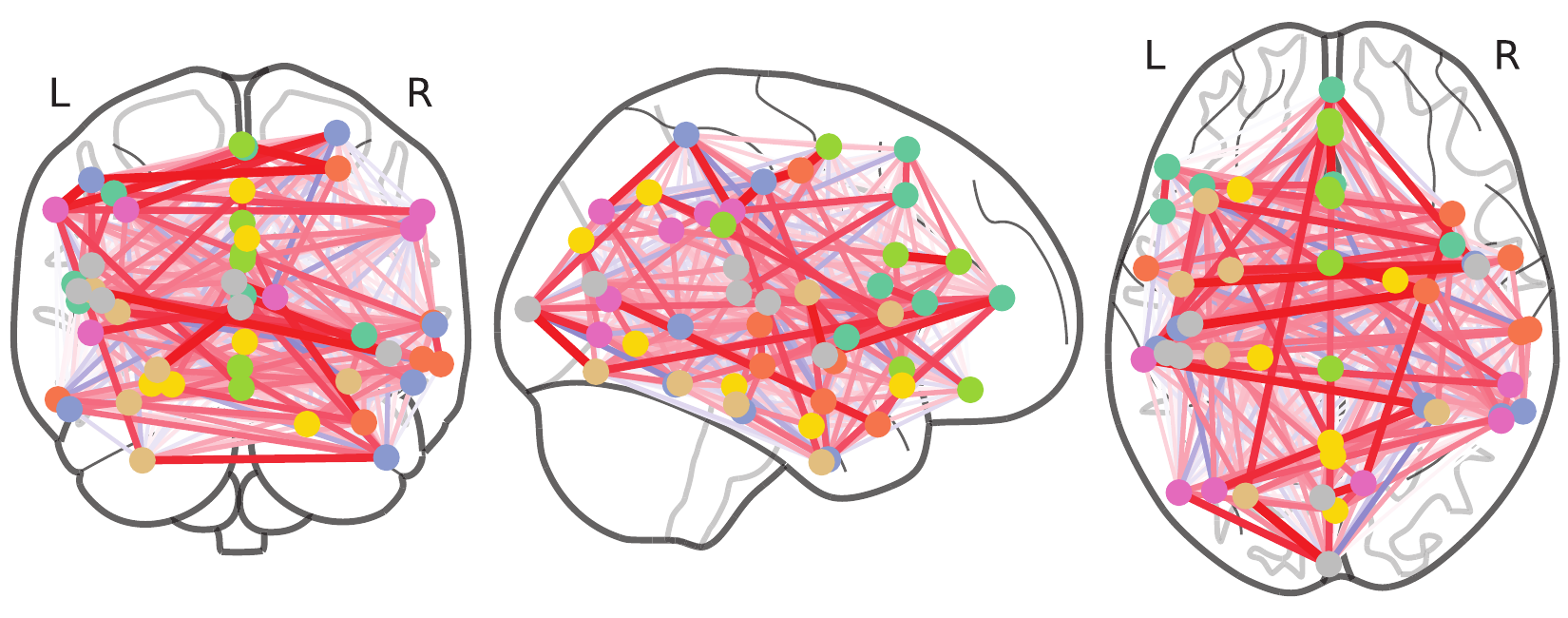}
	\caption{Patients using PLA-GGMs}
\end{figure*}
\begin{figure*}[h]
	\centering
	\includegraphics[scale=0.8]{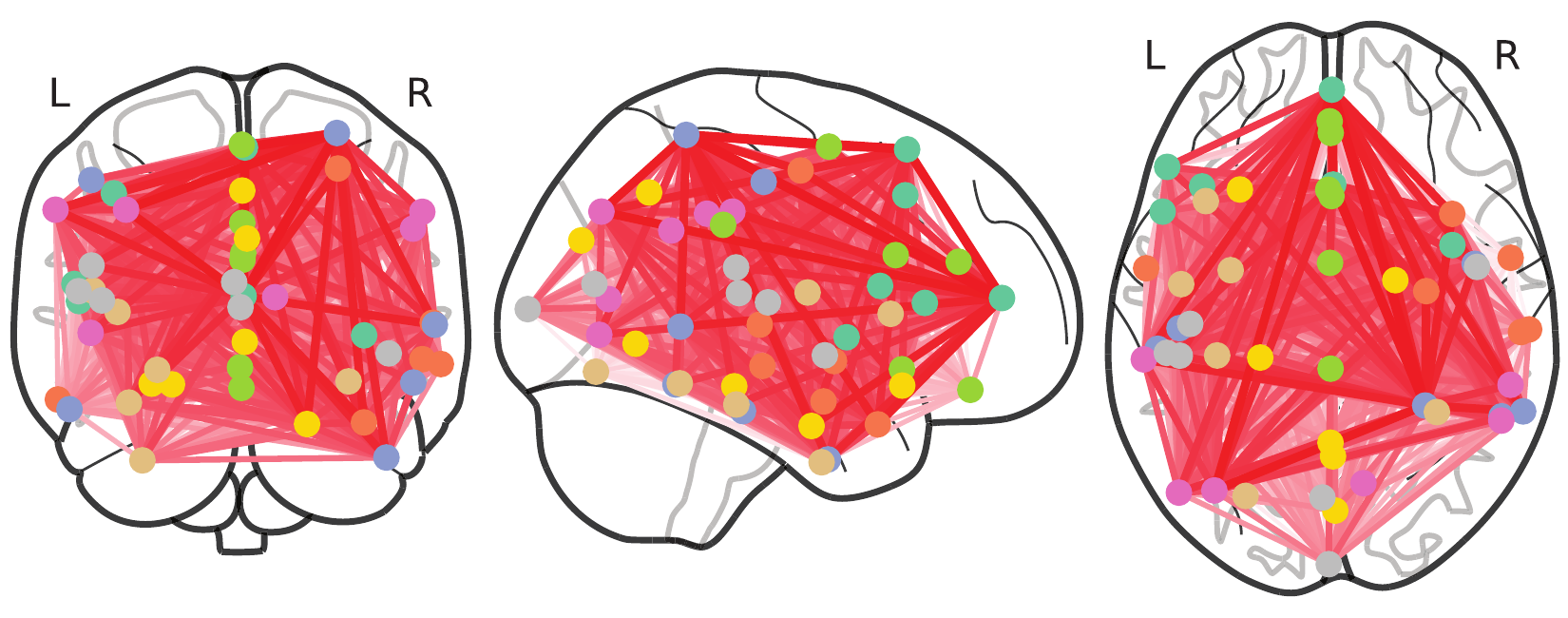}
	\caption{Controls using LR-GGMs}
\end{figure*}
\begin{figure*}[h]
	\centering
	\includegraphics[scale=0.8]{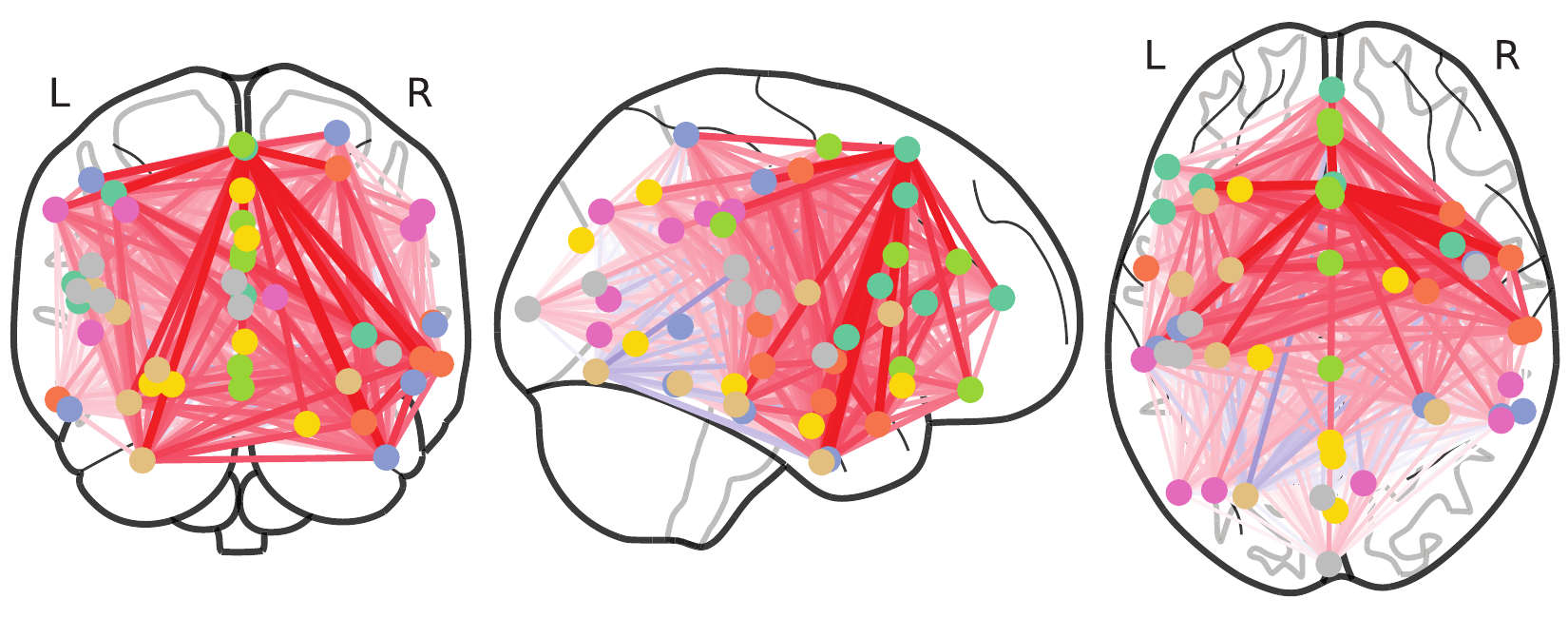}
	\caption{Patients using LR-GGMs}
\end{figure*}

\subsection*{Schizophrenia Diagnosis using Different $\mathds{1}_{\left\{ \abs{g}\geq g^* \right\}}$'s}
We conduct the analysis in Section~\ref{sec:diagnosis} using different $\mathds{1}_{\left\{ \abs{g}\geq g^* \right\}}$'s. Specifically, we consider the function $1- \exp(-kx^2)/2$ using $k=144, 150$. The achieved accuracy using the parameter selected by the 10-fold cross validation and AIC are reported in Figure~\ref{fig:1-function}. The performance of PLA-GGMs is not hugely affected when selecting $\mathds{1}_{\left\{ \abs{g}\geq g^* \right\}}$ in a reasonable range, which is consistent with our analysis in Theorem~\ref{thm:sparsistency}. Note that, if we select $k$ too large, the PPL method will be not applicable. The reason is that a large $k$ corresponds to a small $g^*$, and will induce few non-confounded samples observed. As a result, $\left(
\bD_{ij}^\top \bW_i \bD_{ij}
\right)$ will be singular. In practice, if we use a relative large $g^*$ corresponding to a small $k$, \eqref{eq:asm-pla} will tend to be like $\bR(g) = 0$ used in CON-GGMs and LR-GGMs. 

\begin{figure*}[t]
	\centering
	\begin{subfigure}{0.33\textwidth}
		\centering
		\includegraphics[scale=0.4]{./support/aucaic.pdf}
		\caption{$k=100$}
	\end{subfigure}
	\centering
	\begin{subfigure}{0.33\textwidth}
		\centering
		\includegraphics[scale=0.4]{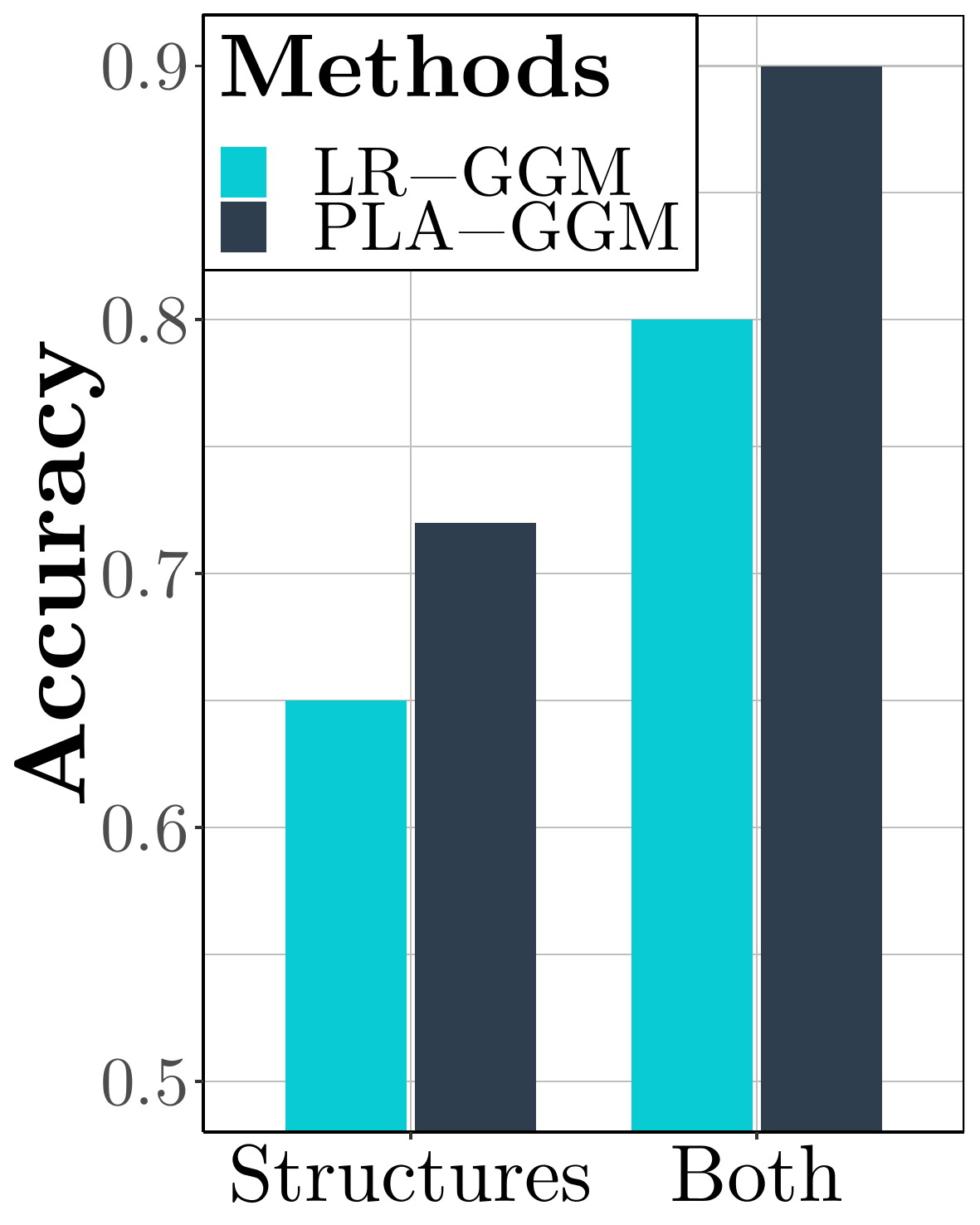}
		\caption{$k = 144$}
	\end{subfigure}
	\centering
	\begin{subfigure}{0.33\textwidth}
		\centering
		\includegraphics[scale=0.4]{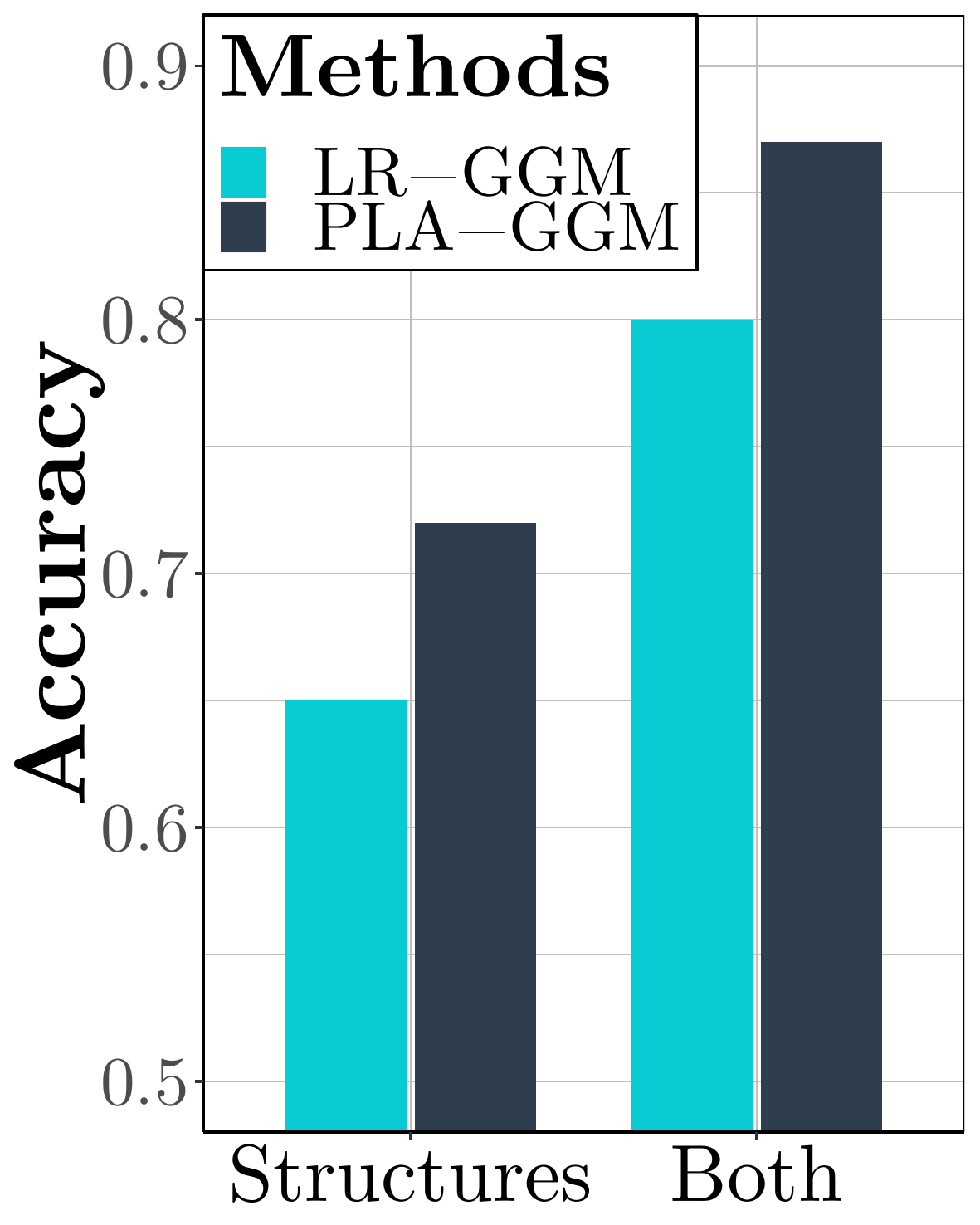}
		\caption{$k = 150$}
	\end{subfigure}
	\centering
	\caption{Diagnosis using different $\mathds{1}_{\left\{ \abs{g}\geq g^* \right\}}$'s. }
	\label{fig:1-function}
\end{figure*}

\end{document}